\newif\ifarxiv
 \arxivtrue

\documentclass{article} 

\usepackage{fullpage}
\usepackage{amssymb, mathtools, bbm, amsthm, natbib, xcolor,hyperref, authblk}
\usepackage[ruled]{algorithm2e} 

\allowdisplaybreaks

\usepackage{enumerate}

\DeclareMathOperator*{\E}{\mathbb{E}}
\newcommand{\CCV}{\mathrm{CCV}}

\newcommand{\CBR}{\mathrm{CBR}}
\newcommand{\Reg}{\mathrm{Reg}}
\newcommand{\ext}{\mathrm{ext}}
\newcommand{\swap}{\mathrm{swap}}

\newcommand{\1}{\mathbbm{1}}
\newcommand{\bc}{\mathbf{c}}

\newcommand{\adapt}{\mathrm{adapt}}
\newcommand{\dyn}{\mathrm{dyn}}

\newcommand{\cA}{\mathcal{A}}

\newcommand{\cC}{\mathcal{C}}

\newcommand{\cE}{\mathcal{E}}
\newcommand{\cF}{\mathcal{F}}

\newcommand{\cN}{\mathcal{N}}

\newcommand{\cS}{\mathcal{S}}

\newcommand{\cU}{\mathcal{U}}

\newcommand{\cX}{\mathcal{X}}
\newcommand{\cY}{\mathcal{Y}}

\DeclareMathOperator*{\argmax}{arg\,max}
\DeclareMathOperator*{\argmin}{arg\,min}

\newcommand{\bR}{\mathbb{R}}
\newcommand{\bN}{\mathbb{N}}
\newtheorem{definition}{Definition}

\newtheorem{lemma}{Lemma}
\newtheorem{theorem}{Theorem}
\newtheorem{remark}{Remark}
\newtheorem{corollary}{Corollary}

\newtheorem{assumption}{Assumption}

\newcommand{\fea}{\mathrm{fea}}
\newcommand{\infea}{\mathrm{inf}}

\usepackage[dvipsnames]{xcolor}

\begin{document}

\title{Dynamic Regret Bounds for Online Omniprediction with Long Term Constraints}

\author[1]{Yahav Bechavod}
\author[2]{Jiuyao Lu}
\author[1]{Aaron Roth}
\affil[1]{Department of Computer and Information Sciences, University of Pennsylvania}
\affil[2]{Department of Statistics and Data Science, University of Pennsylvania}

\date{\today}

\newcommand{\fix}{\marginpar{FIX}}
\newcommand{\new}{\marginpar{NEW}}

\maketitle

\begin{abstract}
We present an algorithm guaranteeing dynamic regret bounds for online omniprediction with long term constraints. The goal in this recently introduced problem is for a learner to generate a sequence of predictions which are broadcast to a collection of downstream decision makers. Each decision maker has their own utility function, as well as a vector of constraint functions, each mapping their actions and an adversarially selected state to reward or constraint violation terms.  The downstream decision makers select actions ``as if'' the state predictions are correct, and the goal of the learner is to produce predictions such that all downstream decision makers choose actions that give them worst-case utility guarantees while minimizing worst-case constraint violation. Within this framework, we give the first algorithm that obtains simultaneous \emph{dynamic regret} guarantees for all of the agents --- where regret for each agent is measured against a potentially changing sequence of actions across rounds of interaction, while also ensuring vanishing constraint violation for each agent. Our results do not require the agents themselves to maintain any state --- they only solve one-round constrained optimization problems defined by the prediction made at that round.
\end{abstract}

\section{Introduction} \label{sec:introduction}

In the problem of learning with long term constraints, there is a decision maker with an action space $\cA$ and an adversary with an outcome space $\cY$. The learner has a utility function $u:\cA\times \cY\rightarrow [0,1]$, and a vector-valued constraint function $c:\cA\times \cY \rightarrow [-1,1]^d$. In rounds $t = 1,\ldots,T$, a \emph{learner} chooses actions $a_t \in \cA$ and an \emph{adversary} chooses outcomes $y_t \in \cY$. The learner then obtains utility $u(a_t,y_t)$, and suffers a marginal constraint increment $c(a_t,y_t)$. The goal of the learner is to satisfy all of the constraints marginally (up to a vanishing regret term) --- i.e. to guarantee that for all sequences of outcomes:
$$\CCV(1:T) \doteq \max_{j} \sum_{t=1}^T c_j(a_t,y_t) \leq o(T),$$
while simultaneously guaranteeing some notion of \emph{regret} to the best action in some benchmark class. It has been known since \cite{MannorTY09} that in adversarial settings it is not possible to compete against the best fixed action in hindsight that satisfies the constraints \emph{marginally}. 
Instead, the  standard benchmark in this literature is the set of actions that in hindsight satisfy the realized constraints \emph{every round}: $$\cA_{1:T}^\bc = \{a\in \cA : c_j(a,y_t) \leq 0   \ \textrm{for every } t \in [T] \ \textrm{and } j \in [J]\}.$$
The corresponding standard goal in this literature (see e.g. \citet{Sun17safety,Castiglioni22,qiu2023gradient,sinha2024optimal}) is to minimize external regret with respect to this benchmark:
$$\Reg_\ext(1:T) \doteq \max_{a \in \cA_{1:T}^\bc} \sum_{t=1}^T \left(u(a,y_t) - u(a_t,y_t) \right) \leq o(T).$$
A more ambitious goal, studied by a sub-thread of this literature, is to compete with a \emph{changing benchmark sequence of actions}, so long as the benchmark sequence does not change too quickly \citep{ChenLG17,chen2018,chen2018bandit,CaoL19,vaze2022dynamic,liu2022,lekeufack2024optimistic}. This is called a \emph{dynamic regret} benchmark. For continuous action spaces, there are a variety of ways to measure ``change'', but we state here a version for discrete categorical action spaces, which is the focus of our paper. First we define a richer benchmark that allows for \emph{changing} sequences of benchmark actions that satisfy the constraints at each round.
$$\cA_{1:T}^\textrm{dyn} = \{\vec{a}\in \cA^T : c_j(\vec{a}_t,y_t) \leq 0   \ \textrm{for every } t \in [T] \ \textrm{and } j \in [J]\}.$$
For a given sequence of actions $\vec{a} \in \cA_{1:T}^\textrm{dyn}$, we write $\Delta (\vec{a}) = |\{t : \vec{a}_t \neq \vec{a}_{t+1}\}|$ for the number of times the action changes in the benchmark sequence. The goal is to obtain diminishing dynamic regret:
$$\max_{\vec{a} \in \cA_{1:T}^\textrm{dyn}}\sum_{t=1}^T \left(u(\vec{a}_t,y_t)-u(a_t,y_t)\right) - \Reg_{\ext-\dyn}(\vec a) \leq 0.$$
Here we want the ``dynamic regret bound'' $\Reg_{\ext-\dyn}(\vec a)$ to be $o(T)$ for all $\vec{a}$ such that $\Delta(\vec{a}) \leq o(T)$.   

Dynamic regret bounds are a very natural objective to strive for. The reason that we study a sequential adversarial environment is that we expect the environment to change (in potentially unpredictable ways) over time. In such environments, we naturally expect the optimal decision to also change over time, and in the constrained optimization setup, we might worry that the static benchmark $\cA_{1:T}^\bc$ is empty, even if there are feasible actions at every time step. Hence, the literature on online learning with long term constraints studying external regret generally make a rather strong assumption regarding the existence of a \emph{single action} that satisfies all of the constraints across \emph{all rounds}. Dynamic regret bounds, on the other hand, overcome this concern for slowly changing environments, by allowing for changes within the benchmark sequence while only requiring \emph{local feasibility}  --- that the action we compare to on each round satisfies the constraints of that specific round (and not across all rounds of the interaction).


While the literature on online learning with long term constraints generally couples the problem of predicting outcomes $y_t$ and choosing actions $a_t$ by focusing on algorithms for a \emph{single decision maker},  \cite{BLR25} recently introduced the \emph{omniprediction} variant of this problem (c.f. \cite{gopalan2021omnipredictors}) in which a single centralized learner broadcasts predictions $p_t$ each round for the outcome $y_t$, and then multiple downstream decision makers (who differ in their utility and constraint functions) choose actions as simple functions of $p_t$. The goal is to make predictions that simultaneously guarantee all such decision makers worst-case regret and constraint violation bounds. \cite{BLR25} shows how to obtain this with respect to the benchmark class $\cA_{1:T}^\bc$, but their algorithm does not extend to give dynamic regret bounds. In this work we show how to make predictions in a way that gives  dynamic regret bounds (in fact stronger dynamic \emph{swap} regret bounds) simultaneously for many downstream decision makers. A detailed discussion of additional related work is deferred to Appendix \ref{app:related-work}.



\subsection{Our Results}

\paragraph{Better Subsequence Regret Bounds}

Previous work on online omniprediction with long term constraints gave regret bounds that held not just marginally over the whole sequence, but simultaneously on an arbitrarily specified collection of \emph{subsequences} of it. In principle dynamic regret bounds can be extracted from subsequence regret bounds like this, by taking the set of subsequences to be the set of all $\approx T^2$ contiguous intervals in $\{1,\ldots,T\}$. Unfortunately  the bounds obtained by \citet{BLR25} depend \emph{linearly} on the number of specified subsequences, which does not yield nontrivial dynamic regret bounds. Our main contribution is a new algorithm giving regret and constraint violation guarantees on arbitrary  subsequences, with a dependence on the number of subsequences scaling only \emph{logarithmically}. Dynamic regret bounds fall out as a special case.

\paragraph{Stronger Notions of Regret}
In fact, since our subsequence regret bounds are stronger \emph{swap regret} bounds, the dynamic regret bounds we obtain are stronger than those that have been previously studied in the literature on learning with long term constraints: we give bounds on what we call \emph{dynamic swap regret}. Our new  benchmark allows each decision maker to compete with a sequence of actions that results from applying a \emph{swap function} remapping the decision maker's realized actions to alternatives. We allow the swap function itself to change with time. The traditional notion of dynamic regret is the special case in which these swap functions are constant valued.

\paragraph{Easier Implementation for Downstream Agents}

Finally, \citet{BLR25} obtained their results by requiring downstream agents to map predictions to actions using an elimination-based algorithm, which required all agents to actively maintain state --- the set of actions that had not yet violated any of their constraints. Our algorithm allows downstream decision makers to map predictions to actions in an entirely stateless way: they simply evaluate both their constraint function and objective function as if our predictions were correct, and take the action that solves the resulting one-round constrained optimization problem.

\section{Model and Preliminaries} \label{sec:preliminaries}
The model largely follows the framework for online omniprediction with long-term constraints introduced in \cite{BLR25}.
Let $\cX$ denote the feature space and $\cY$ denote the outcome/label space. Throughout, we consider $\cY = [0,1]^d$. 
We consider a set of agents $\cN$ with an arbitrary action space $\cA$. Each agent is equipped with a tuple $(u,c_1,\ldots,c_J)$, which includes a utility function $u: \cA \times \cY \to [0,1]$ and $J$ constraint functions $\{c_j: \cA \times \cY \to [-1,1]\}_{j \in [J]}$\footnote{Sometimes online adversarial learning problems are described by an adversary choosing a different utility and/or constraint function at each step. This is equivalent to having a fixed state-dependent utility function/constraint functions, and having an adversary choose state.}.
We also sometimes write the constraint functions as a single vector valued function $\bc = (c_1,\ldots,c_J)$. 
Since agents are uniquely defined by their corresponding tuples, we treat agents and their tuples interchangeably.
We assume that the utility functions are linear and Lipschitz-continuous in $y$. 
\begin{assumption}
    Fix any utility function $u: \cA \times \cY \to [0,1]$. We assume that for every action $a \in \cA$, $u(a, y)$ is linear in $y$, i.e. $u(a, k_1 y_1 + k_2 y_2) = k_1 u(a, y_1) + k_2 u(a, y_2)$ for all $k_1,k_2 \in \bR$, $y_1,y_2 \in \cY$. Moreover, we assume there exists a universal constant $L_\cU$ such that any utility function $u(a,y)$ is $L_\cU$-Lipschitz in $y$ in the $\ell_\infty$ norm: for any $a \in \cA$ and $y_1,y_2 \in \cY$, $|u(a, y_1) - u(a, y_2)| \leq L_\cU\|y_1 - y_2\|_\infty$. 
\end{assumption}

In addition to the above assumption from \cite{BLR25}, we also assume that the constraint functions are linear and Lipschitz-continuous in $y$. This will enable the purely prediction-based decision rule we introduce later, where agents select actions that are predicted to be feasible without needing to track historical constraint violations as was required in \cite{BLR25}.
\begin{assumption}
    Fix any constraint function $c_j: \cA \times \cY \to [-1,1]$. We assume that for every action $a \in \cA$, $c_j(a, y)$ is linear in $y$, i.e. $c_j(a, k_1 y_1 + k_2 y_2) = k_1 c_j(a, y_1) + k_2 c_j(a, y_2)$ for all $k_1,k_2 \in \bR$, $y_1,y_2 \in \cY$. Moreover, we assume there exists a universal constant $L_\cC$ such that any constraint function $c_j(a,y)$ is $L_\cC$-Lipschitz in $y$ in the $\ell_\infty$ norm: for any $a \in \cA$ and $y_1,y_2 \in \cY$, $|c_j(a, y_1) - c_j(a, y_2)| \leq L_\cC\|y_1 - y_2\|_\infty$. 
\end{assumption}

\begin{remark}
    For simplicity we assume that the utility functions and constraint functions are \emph{linear} in $y$, but we can equally well handle the case in which the utility functions are \emph{affine} in $y$, as we can augment the label space with an additional coordinate that takes constant value 1. This preserves the convexity of the label space and allows for arbitrary constant offsets in the utility/constraint of each action. Assuming linear/affine utility functions is only more general than the standard assumption that decision makers are \emph{risk neutral} in the sense that in the face of randomness, decision makers act to maximize their expected utility. If $\cY$ represents the set of probability distributions over outcomes, any risk neutral decision maker has a linear utility function by linearity of expectation. 
\end{remark}

We take the role of an online/sequential forecaster producing predictions that will be consumed by agents. We consider the following repeated interaction between a forecaster, agents, and an adversary. In every round $t \in [T]$:
\begin{enumerate}[(1)]
    \item The adversary selects a feature vector $x_t \in \cX$ and a distribution over outcomes $Y_t \in \Delta \cY$;
    \item The forecaster observes the feature $x_t$, produces a distribution over predictions $\pi_t\in\Delta\cY$, from which a prediction $p_t \in \cY$ is sampled; 
    \item Each agent chooses an action $a_t$ as a function of the prediction $p_t$ and the history;
    \item The adversary reveals an outcome $y_t \sim Y_t$, and the agent obtains utility $u(a_t, y_t)$ and the constraint loss vector $\{c_j(a_t, y_t)\}_{j \in [J]}$.
\end{enumerate}

We will focus on performance over a collection of subsequences $\cS$, where each subsequence $S \in \cS$ is a subset of $[T]$. 
These subsequences need not be fixed in advance but can be defined dynamically. A subsequence $S \in \cS$ is generally characterized by an indicator function $h_S: [T] \times \cX \to \{0,1\}$. For any round $t \in [T]$, the round is part of the subsequence $S$ if and only if $h_S(t, x_t) = 1$. This flexible definition allows subsequences to be based on the round index $t$, the feature $x_t$, or both. 

Agents aim to maximize their cumulative utilities over every subsequence in $\cS$: $\sum_{t \in S} u(a_t,y_t)$
while minimizing their cumulative constraint violation (CCV) over every subsequence in $\cS$:
\[
    \CCV(S) \coloneqq \max_{j \in [J]} \sum_{t \in S} c_j(a_t,y_t) \le o(|S|).
\]
We treat utility maximization as an objective and cumulative constraint violation as a requirement: $\CCV(S)$ must be sublinear in the length of the subsequence, $|S|$. 

We measure performance against different benchmark classes. The fundamental building block for our benchmark classes is the set of actions that are feasible at a specific round $t \in [T]$ with a margin of $\lambda \ge 0$:
\begin{align*}
    \cA_{t}^{\bc,\lambda} = \left\{ a \in \cA: c_j(a,y_t) \le -\lambda \; \text{ for every } j \in [J] \right\}.
\end{align*}
The margin $\lambda$ parameterizes the difficulty of the benchmarks; a smaller $\lambda$ yields a larger and thus more competitive set of actions\footnote{Assuming the existence of such a strongly feasible action is known in the literature on learning with long term constraints as Slater's condition \citep{neelyYu2017,ChenLG17,CaoL19,YuNeely2020,Castiglioni22}). We later additionally give bounds that hold without making this assumption.}.
We note that this margin can be set independently for each subsequence $S \in \cS$, and our final results will be achieved by choosing its value as a function of the subsequence length $|S|$. 

In the literature on learning with long-term constraints, a standard benchmark is the set of actions that are feasible at every round. We generalize this to our multi-subsequence framework by requiring this condition to hold throughout a given subsequence $S \in \cS$:
\begin{align*}
    \cA_{S}^{\bc,\lambda} = \cap_{t \in S} \cA_{t}^{\bc,\lambda} = \left\{ a \in \cA: c_j(a,y_t) \le -\lambda \; \text{ for every } t \in S \text{ and } j \in [J] \right\}.
\end{align*}

Competing with the best fixed action from this class in hindsight leads to the notion of \emph{constrained external regret}.
\begin{definition}[Constrained External Regret over Subsequence $S$] \label{def:constrained-external-regret}
    Fix an agent with a utility function $u: \cA \times \cY \to [0,1]$ and a constraint function $\bc: \cA \times \cY \to [-1,1]^J$. Fix a subsequence $S \subseteq [T]$. 
    For a sequence of actions $a_1,\ldots,a_T$ and outcomes $y_1,\ldots,y_T$, the agent's constrained external regret over the subsequence $S$ is:
    \[
        \Reg_\ext(u,\bc,\lambda,S) = \max_{a \in \cA_{S}^{\bc,\lambda}}\sum_{t \in S} \left( u(a,y_t) - u(a_t,y_t) \right).
    \]
\end{definition}

We will also compete with a more demanding benchmark based on action modification rules. For the benchmark class $\cA_S^{\bc,\lambda}$, an action modification rule is any function $\phi: \cA \to \cA_S^{\bc,\lambda}$ that consistently maps an agent's actions to alternatives within the benchmark class. Competing with the best such rule in hindsight leads to the notion of \emph{constrained swap regret}. This is a stronger notion than constrained external regret, as constrained external regret can be viewed as a special case where the action modification rule is restricted to being a constant function.
\begin{definition}[Constrained Swap Regret over Subsequence $S$] \label{def:constrained-swap-regret}
    Fix an agent with a utility function $u: \cA \times \cY \to [0,1]$ and a constraint function $\bc: \cA \times \cY \to [-1,1]^J$. Fix a subsequence $S \subseteq [T]$. 
    For a sequence of actions $a_1,\ldots,a_T$ and outcomes $y_1,\ldots,y_T$, the agent's constrained swap regret over the subsequence $S$ is:
    \[
        \Reg_\swap(u,\bc,\lambda,S) = \max_{\phi: \cA \to \cA_S^{\bc,\lambda}} \sum_{t \in S} \left( u(\phi(a_t),y_t) - u(a_t,y_t) \right).
    \]
\end{definition}

\ifarxiv
A particularly important special case of our framework is adaptive regret, which guarantees low regret simultaneously over all contiguous intervals. This is achieved by setting the collection of subsequences to be $\cS = \{[t_1,t_2] : 1 \le t_1 \le t_2 \le T\}$.
This gives rise to the following notions of \emph{constrained external adaptive regret} and \emph{constrained swap adaptive regret}.

\begin{definition}[Constrained External Adaptive Regret]
    Fix an agent with a utility function $u: \cA \times \cY \to [0,1]$ and a constraint function $\bc: \cA \times \cY \to [-1,1]^J$. 
    For a sequence of actions $a_1,\ldots,a_T$ and outcomes $y_1,\ldots,y_T$, the agent's constrained external adaptive regret is:
    \[
        \Reg_{\ext-\adapt}(u,\bc,\lambda) = \max_{1 \le t_1 \le t_2 \le T} \max_{a \in \cA_{[t_1,t_2]}^{\bc,\lambda}}\sum_{t=t_1}^{t_2} \left( u(a,y_t) - u(a_t,y_t) \right).
    \]
\end{definition}

\begin{definition}[Constrained Swap Adaptive Regret]
    Fix an agent with a utility function $u: \cA \times \cY \to [0,1]$ and a constraint function $\bc: \cA \times \cY \to [-1,1]^J$. 
    For a sequence of actions $a_1,\ldots,a_T$ and outcomes $y_1,\ldots,y_T$, the agent's constrained swap adaptive regret:
    \[
        \Reg_{\swap-\adapt}(u,\bc,\lambda) = \max_{1 \le t_1 \le t_2 \le T} \max_{\phi: \cA \to \cA_{[t_1,t_2]}^{\bc,\lambda}} \sum_{t=t_1}^{t_2} \left( u(\phi(a_t),y_t) - u(a_t,y_t) \right).
    \]
\end{definition}
\fi

While adaptive regret provides a powerful guarantee over all contiguous intervals, a different but related goal is to measure performance against a dynamic benchmark path that changes over time, known as dynamic regret. We now introduce the external and swap versions of this guarantee under our framework with long-term constraints.

The benchmarks for \emph{constrained external dynamic regret} are changing sequences of benchmark actions that satisfy the constraints (with a fixed margin of $\lambda$) at each round:
\begin{align*}
    \vec\cA_{1:T}^{\bc,\lambda} = \prod_{t=1}^T \cA_t^{\bc,\lambda} = \{\vec a \in \cA^T : c_j(\vec a_t,y_t) \leq -\lambda   \ \textrm{for every } t \in [T] \ \textrm{and } j \in [J]\}.
\end{align*}
The complexity of any such benchmark sequence $\vec a \in \vec\cA_{1:T}^{\bc,\lambda}$ is measured by the number of times the action changes:
$\Delta (\vec{a}) = |\{t \in [T-1] : \vec a_t \neq \vec a_{t+1}\}|$.

\begin{definition}[Constrained External Dynamic Regret]
    Fix an agent with a utility function $u: \cA \times \cY \to [0,1]$ and a constraint function $\bc: \cA \times \cY \to [-1,1]^J$. For a sequence of actions $a_1,\ldots,a_T$ and outcomes $y_1,\ldots,y_T$, the agent's constrained external dynamic regret against a benchmark sequence of actions $\vec a \in \cA^T$ is:
    \[
        \Reg_{\ext-\dyn}(u,\vec a) = \sum_{t=1}^{T} \left( u(\vec a_t,y_t) - u(a_t,y_t) \right).
    \]
    The goal is to guarantee that $\Reg_{\ext-\dyn}(u,\vec a)$ is $o(T)$ for any benchmark sequence $\vec a$ that is (1) dynamically feasible, $\vec a \in \vec\cA_{1:T}^{\bc,\lambda}$, and (2) has a sublinear number of changes, $\Delta(\vec a) = o(T)$.
\end{definition}

Next, we define the more powerful swap-based counterpart. The benchmarks for \emph{constrained swap dynamic regret} are changing sequences of action modification rules $\vec\phi \in (\cA^\cA)^T$. 
The complexity of any such benchmark sequence is likewise measured by the number of times the action modification rule changes:
$\Delta (\vec\phi) = |\{t \in [T-1]: \vec \phi_t \neq \vec \phi_{t+1}\}|$.
A sequence with $\Delta(\vec\phi)$ changes partitions the time horizon $[1, T]$ into $\Delta(\vec\phi) + 1$ contiguous intervals, on each of which the rule is fixed.
We only compete with sequences where, for each interval of constancy, the fixed action modification rule maps to the set of actions that are feasible throughout that entire interval.


\begin{definition}[Constrained Swap Dynamic Regret]
    Fix an agent with a utility function $u: \cA \times \cY \to [0,1]$ and a constraint function $\bc: \cA \times \cY \to [-1,1]^J$. For a sequence of actions $a_1,\ldots,a_T$ and outcomes $y_1,\ldots,y_T$, the agent's constrained swap dynamic regret against a benchmark sequence of action modification rules $\vec\phi \in (\cA^\cA)^T$ is:
    \[
        \Reg_{\swap-\dyn}(u,\vec\phi) = \sum_{t=1}^{T} \left( u(\vec\phi_t(a_t),y_t) - u(a_t,y_t) \right).
    \]
    The goal is to guarantee that $\Reg_{\swap-\dyn}(u,\vec\phi)$ is $o(T)$ for any benchmark sequence $\vec\phi$ that satisfies two properties: (1) It is piecewise feasible. Let the sequence have change points that partition $[1, T]$ into intervals $I_0, \ldots, I_{\Delta(\vec{\phi})}$. The action modification rule on each interval $I_k$ is the constant rule $\psi_k: \cA \to \cA_{I_k}^{\bc,\lambda}$. (2) It has a sublinear number of changes, $\Delta(\vec \phi) = o(T)$.
\end{definition}

Any dynamic benchmark sequence (either $\vec a$ or $\vec\phi$) with $\Delta$ changes naturally partitions the time horizon $[1, T]$ into $\Delta + 1$ contiguous intervals based on its change points. Within each of these intervals, the benchmark is fixed. Since an adaptive regret guarantee ensures low regret over \emph{all} possible intervals, it also ensures low regret over this specific partition. The total dynamic regret can therefore be bounded by summing the regret bounds over these $\Delta + 1$ intervals, hence a low adaptive regret bound implies a low dynamic regret bound.

We will first focus on the strictly feasible benchmark classes with a margin of $\lambda = \tilde\Omega(1/\sqrt{T})$, and obtain $\tilde O(\sqrt{T})$ regret and $\tilde O(\sqrt{T})$ cumulative constraint violation;
We will then apply similar techniques to handle the nominally feasible benchmark classes with zero margin ($\lambda = 0$), and obtain $\tilde O(T^{2/3})$ regret and $\tilde O(T^{2/3})$ cumulative constraint violation.

We note that all benchmark classes discussed are agent-specific as they depend on agents' constraint functions. All guarantees we provide will be stated under the assumption that the corresponding benchmark class is non-empty.

\section{Proposed Approach and Main Guarantees}

\subsection{A Purely Prediction-Based Decision Rule}
We propose a stateless decision rule for the agents that is purely based on the forecaster's predictions. At each round $t \in [T]$, each agent acts as if the prediction $p_t$ were accurate, and chooses the action that offers the highest predicted utility among all actions predicted to be feasible. We say that agents play \emph{constrained best responses} to the predictions.

\begin{definition}[Constrained Best Response] \label{def:constrained-best-response}
    Fix a utility $u: \cA \times \cY \to [0,1]$, $J$ constraints $\{c_j: \cA \times \cY \to [-1,1]\}_{j \in [J]}$, and a prediction $p \in \cY$. The constrained best response to $p$ according to $u$ and $\{c_j\}_{j \in [J]}$, denoted as $\CBR^{u,\bc}(p)$, is the solution to the constrained optimization problem:
    \begin{align*}
        \underset{a \in \cA}{\text{maximize}} &\quad u(a,p_t) \\
        \text{subject to} &\quad c_j(a,p_t) \le 0 \; \text{ for every } j \in [J]
    \end{align*}
\end{definition}

The agent can obtain $\CBR^{u,\bc}(p_t)$ in two steps:
\begin{enumerate}[(1)]
    \item 
    The agent first discards actions that are infeasible according to the prediction.

    For each constraint $j \in [J]$, the set of actions predicted to violate that constraint is denoted as:
    $$\widehat{\cA}_t^{c_j,\infea} = \left\{ a \in \cA: c_j(a,p_t) > 0 \right\}.$$
    
    The agent discards actions that are predicted to violate any of the $J$ constraints, i.e.,
    $$\widehat{\cA}_t^{\bc,\infea} = \cup_{j \in [J]} \widehat{\cA}_t^{c_j,\infea} = \left\{ a \in \cA: \exists j \in [J], c_j(a,p_t) > 0 \right\}.$$
    
    The retained actions that are predicted to be feasible are denoted as $$\widehat{\cA}_t^{\bc,\fea} = \left\{ a \in \cA: \forall j \in [J], c_j(a,p_t) \le 0 \right\}.$$
    \item
    The agent then chooses an action from the retained action set $\widehat{\cA}_t^{\bc,\fea}$ that maximizes the utility function according to the prediction, i.e.,
    \begin{align*}
        a_t = \argmax_{a \in \widehat{\cA}_t^{\bc,\fea}} u(a,p_t).
    \end{align*}
\end{enumerate}

If none of the actions are predicted to be feasible, i.e., $\widehat\cA_t^{\bc,\fea} = \emptyset$, the agent can choose any arbitrary action from $\cA$. As we will formally prove, our predictions ensure this special case rarely occurs, and hence its influence on the cumulative constraint violation and regret is negligible.

\subsection{Conditionally Unbiased Predictions}
To ensure that our predictions are trustworthy so that treating them as the truth is a sensible choice for the agents, we need the predictions to be unbiased --- not only marginally, but also conditionally on various subsequences. Notably, we define the following notions of conditional unbiasedness. The first is a ``constrained'' variant of decision calibration as defined in \cite{noarov2023highdimensional}, which itself is a strengthening of an earlier notion of decision calibration due to \cite{zhao2021calibrating}

\begin{definition}[$(\cN,\cS,\alpha)$-Decision Calibration] \label{def:decision-calibration}
    Let $\cS$ be a collection of subsequences. Let $\cN$ be a set of agents, where each agent is equipped with a utility function $u: \cA \times \cY \to [0,1]$ and $J$ constraint functions $\{c_j: \cA \times \cY \to [-1,1]\}_{j \in [J]}$. We say that a sequence of predictions $p_1,\ldots,p_T$ is $(\cN,\cS,\alpha)$-decision calibrated with respect to a sequence of outcomes $y_1,\ldots,y_T$ if for every $S \in \cS$, $a \in \cA$, and $(u,\bc) \in \cN$:
    \[
        \left\| \sum_{t=1}^T \1\left[t \in S, \CBR^{u,\bc}(p_t) = a \right] (p_t - y_t) \right\|_\infty \leq \alpha(T^{u,\bc,S}(a)),
    \]
    where $T^{u,\bc,S}(a) = \sum_{t=1}^T \1\left[t \in S, \CBR^{u,\bc}(p_t) = a\right]$.
\end{definition}

Decision calibration guarantees that forecasts are unbiased conditional on the decisions of the downstream decision makers. Infeasibility calibration, defined next, requires that the predictions  be unbiased conditional on each action for each downstream decision maker being predicted to be infeasible.

\begin{definition}[$(\cN,\cS,\beta)$-Infeasibility Calibration] \label{def:infeasibility-calibration}
    Let $\cS$ be a collection of subsequences. Let $\cN$ be a set of agents, where each agent is equipped with a utility function $u: \cA \times \cY \to [0,1]$ and $J$ constraint functions $\{c_j: \cA \times \cY \to [-1,1]\}_{j \in [J]}$. We say that a sequence of predictions $p_1,\ldots,p_T$ is $(\cN,\cS,\beta)$-infeasibility calibrated with respect to a sequence of outcomes $y_1,\ldots,y_T$ if for every $S \in \cS$, $a \in \cA$, $(u,\bc) \in \cN$, and $j \in [J]$:
    \[
        \left\| \sum_{t=1}^T \1\left[t \in S, a \in \widehat\cA_t^{c_j,\infea}\right] (p_t - y_t) \right\|_\infty \leq \beta(T^{c_j,S,\infea}(a)),
    \]
    where $T^{c_j,S,\infea}(a) = \sum_{t=1}^T \1\left[t \in S, a \in \widehat\cA_t^{c_j,\infea}\right]$.
\end{definition}


\begin{assumption}
    We assume that $\alpha,\beta: \bR \to \bR$ are concave functions. This will be the case in all the bounds we give; in general, this condition holds for any sublinear error bound $T^r$ for $r<1$.
\end{assumption}

In the sections that follow, We will show how these two unbiasedness constraints lead to bounds on the cumulative constraint violation and regret for all downstream decision makers.

We now address the algorithmic challenge of producing predictions that are decision calibrated and infeasibility calibrated. Our approach builds upon the \textsc{Unbiased-Prediction} algorithm from \citet{noarov2023highdimensional}, which makes conditionally unbiased predictions in the online setting. The algorithm and its guarantees are presented in Appendix \ref{app:unbiased-prediction}; we refer interested readers to the original work for further details.

We will instantiate \textsc{Unbiased-Prediction} to make predictions that simultaneously achieve decision calibration and infeasibility calibration; we will refer to this instantiation as \textsc{Decision-Infeasibility-Calibration}. 
Our guarantees will inherit from the guarantees of \textsc{Unbiased-Prediction}. 

\begin{theorem} \label{thm:unbiased-algorithm}
    Let $\cS$ be a collection of subsequences. Let $\cN$ be a set of agents, where each agent is equipped with a utility function $u: \cA \times \cY \to [0,1]$ and $J$ constraint functions $\{c_j: \cA \times \cY \to [-1,1]\}_{j \in [J]}$.
    There is an instantiation of \textsc{Unbiased-Prediction} \citep{noarov2023highdimensional}
    ---which we call \textsc{Decision-Infeasibility-Calibration}---
    producing predictions $p_1,...,p_T \in \cY$ satisfying that for any sequence of outcomes $y_1,...,y_T \in \cY$, 
    with probability at least $1-\delta$, for any $(u,\bc) \in \cN$, $j \in [J]$, $a \in \cA$, and $S \in \cS$:
        \begin{align*}
            & \left\| \sum_{t=1}^T \1\left[t \in S, \CBR^{u,\bc}(p_t) = a \right] (p_t - y_t) \right\|_\infty \leq O\left( \ln\frac{dJ|\cA||\cN||\cS|T}{\delta} \cdot |S|^{1/4} + \sqrt{\ln\frac{dJ|\cA||\cN||\cS|T}{\delta} \cdot T^{u,\bc,S}(a)} \right), \\
            & \left\| \sum_{t=1}^T \1\left[t \in S, a \in \widehat\cA_t^{c_j,\infea}\right] (p_t - y_t) \right\|_\infty \leq O\left( \ln\frac{dJ|\cA||\cN||\cS|T}{\delta} \cdot |S|^{1/4} + \sqrt{\ln\frac{dJ|\cA||\cN||\cS|T}{\delta} \cdot T^{c_j,S,\infea}(a)} \right).
        \end{align*}
\end{theorem}

\subsection{Theoretical Guarantees}

We begin our analysis with a preliminary lemma. Recall that if the set of predicted feasible actions is empty, the agent plays an arbitrary action. The following lemma uses the infeasibility calibration guarantee to show that for any benchmark action $a \in \cA_S^{\bc,\lambda}$, the number of rounds where it is incorrectly predicted to violate a specific constraint $c_j$ is small. As a result, the number of rounds on which no action is predicted to be feasible is small.

First, we establish some notation for simplicity. Let $g_\beta(x) = x / \beta(x)$. For the specific form of $\beta$ provided by the guarantee in Theorem \ref{thm:unbiased-algorithm}, $g_\beta$ is monotone for $x > 0$. We define its inverse function as $f_\beta = g_\beta^{-1}$.
\begin{lemma} \label{lem:infeasibility-calibration}
    Suppose the benchmark class $\cA_S^{\bc,\lambda}$ is non-empty. If the sequence of predictions $p_1,\ldots,p_T$ is $(\cN,\cS,\beta)$-infeasibility calibrated, then for any agent $(u,\bc) \in \cN$, subsequence $S \in \cS$, benchmark action $a \in \cA_S^{\bc,\lambda}$, and constraint $j \in [J]$, the number of rounds $T^{c_j,S,\infea}(a)$ within $S$ on which $a$ is predicted to violate the $j$-th constraint is bounded by:
    \begin{align*}
        T^{c_j,S,\infea}(a) \le f_\beta(L_\cC/\lambda)
    \end{align*}

    Consequently, the number of rounds within $S$ on which no actions are predicted to be feasible is bounded by:
    \begin{align*}
        \left|\left\{ t \in S: \widehat\cA_t^{\bc,\fea} = \emptyset \right\}\right| \le J f_\beta(L_\cC/\lambda).
    \end{align*}

    In particular, plugging in the guarantee from  Theorem \ref{thm:unbiased-algorithm} yields the following concrete form of $f_\beta(L_\cC/\lambda)$, which holds with probability at least $1-\delta$:
    \begin{align*}
        f_\beta(L_\cC/\lambda) = O\left( \left(\frac{L_\cC |S|^{1/4}}{\lambda} + \frac{L_\cC^2}{\lambda^2} \right) \ln(dJ|\cA||\cN||\cS|T/\delta) \right).
    \end{align*}
\end{lemma}
\newcommand{\proofLemmaInfeasibilityCalibration}{
    On any round $t$ where action $a$ is predicted to violate the $j$-th constraint, we have:
    \begin{align*}
        c_j(a,p_t) > 0
    \end{align*}
    
    Since $a$ is in the benchmark class $\cA_S^{\bc,\lambda}$, we have:
    \begin{align*}
        c_j(a,y_t) \le -\lambda
    \end{align*}
    
    Combining these two facts gives $c_j(a,p_t)-c_j(a,y_t) > \lambda$. Summing this difference over all rounds in $S$ where $a$ is predicted to violate the $j$-th constraint, we get:
    \begin{align*}
        \sum_{t=1}^T \1\left[t \in S, a \in \widehat\cA_t^{c_j,\infea}\right] (c_j(a,p_t)-c_j(a,y_t)) > \lambda T^{c_j,S,\infea}(a)
    \end{align*}
    
    The left-hand side can be bounded using the properties of our predictions. By linearity and $L_\cC$-Lipschitzness of the constraint function, and by $(\cN,\cS,\beta)$-infeasibility calibration, we have that:
    \begin{align*}
        \MoveEqLeft \sum_{t=1}^T \1\left[t \in S, a \in \widehat\cA_t^{c_j,\infea}\right] (c_j(a,p_t)-c_j(a,y_t)) \\
        &= c_j\left(a, \sum_{t=1}^T p_t \1\left[t \in S, a \in \widehat\cA_t^{c_j,\infea}\right]\right) - c_j\left(a, \sum_{t=1}^T y_t \1\left[t \in S, a \in \widehat\cA_t^{c_j,\infea}\right]\right) \\
        &\le L_\cC \left\|\sum_{t=1}^T p_t \1\left[t \in S, a \in \widehat\cA_t^{c_j,\infea}\right] - \sum_{t=1}^T y_t \1\left[t \in S, a \in \widehat\cA_t^{c_j,\infea}\right] \right\|_\infty \\
        &\le L_\cC \beta(T^{c_j,S,\infea}(a))
    \end{align*}

    Combining these inequalities gives the first part of the lemma:    
    \begin{align*}
        \lambda T^{c_j,S,\infea}(a) < L_\cC \beta(T^{c_j,S,\infea}(a))
    \end{align*}

    For the concrete bound, we substitute the explicit form for $\beta$ from Theorem \ref{thm:unbiased-algorithm}:
    \begin{align*}
        \lambda f_\beta(L_\cC / \lambda) = L_\cC O\left( \ln\frac{dJ|\cA||\cN||\cS|T}{\delta} \cdot |S|^{1/4} + \sqrt{\ln\frac{dJ|\cA||\cN||\cS|T}{\delta} \cdot f_\beta(L_\cC / \lambda)} \right)
    \end{align*}

    Solving for $f_\beta(L_\cC / \lambda)$ yields the stated form:
    \begin{align*}
        f_\beta(L_\cC / \lambda) = O\left( \frac{L_\cC |S|^{1/4} \ln(dJ|\cA||\cN||\cS|T/\delta)}{\lambda} + \frac{L_\cC^2 \ln(dJ|\cA||\cN||\cS|T/\delta)}{\lambda^2} \right)
    \end{align*}
}

\newcommand{\proofLemmaEmptySet}{
    Fix any arbitrary benchmark action $a \in \cA_S^{\bc,\lambda}$. If on round $t \in S$, no action is predicted to be feasible, then it must be that $a$ is predicted to be infeasible. Hence, there must exist at least one constraint that $a$ is predicted to violate. As a result, we have:
    \begin{align*}
        \left|\left\{ t \in S: \widehat\cA_t^{\bc,\fea} = \emptyset \right\}\right| &\le \left|\left\{ t \in S: \exists j \in [J], c_j(a,p_t) > 0 \right\}\right| \\
        &\le \sum_{j=1}^J T^{c_j,S,\infea}(a) \\
        &\le J f_\beta(L_\cC / \lambda)
    \end{align*}
}

\ifarxiv
\begin{proof}
    \proofLemmaInfeasibilityCalibration
    \proofLemmaEmptySet
\end{proof}
\else
\fi

\subsubsection{Bounding the Cumulative Constraint Violation}

We now show that the conditional unbiasedness of our predictions guarantees all agents can satisfy the long-term constraints over any subsequences in $\cS$. On a high level, this is because (by definition), downstream agents only play actions that we \emph{predict} to be feasible. On the other hand, our predictions are guaranteed to be unbiased on the subsequence of days on which the downstream agents play each action, and so by linearity of the constraint functions, their cumulative constraint violation cannot be much larger than their predicted cumulative constraint violation (which is non-positive). 
\begin{theorem} \label{thm:CCV}
    Let $\cS$ be a collection of subsequences. 
    Let $\cN$ be a set of agents, where each agent is equipped with a utility function $u: \cA \times \cY \to [0,1]$ and $J$ constraint functions $\{c_j: \cA \times \cY \to [-1,1]\}_{j \in [J]}$. Suppose each agent plays constrained best responses to $p_t$ to choose action $a_t$. Suppose the benchmark class $\cA_S^{\bc,\lambda}$ is non-empty. If the sequence of predictions $p_1,\ldots,p_T$ is $(\cN,\cS,\alpha)$-decision calibrated and $(\cN,\cS,\beta)$-infeasibility calibrated, then the cumulative constraint violation of any agent over any subsequence $S \in \cS$ is bounded by:
    \[
        \CCV(S) \le L_\cC |\cA| \alpha(|S|/|\cA|) + J f_\beta(L_\cC/\lambda).
    \]
    In particular, plugging in the guarantee from  Theorem \ref{thm:unbiased-algorithm} yields the following bound, which holds with probability at least $1-\delta$:
    \[
        \CCV(S) \le O\left( \left( L_\cC |\cA| |S|^{1/4} + L_\cC \sqrt{|\cA||S|} + \frac{J L_\cC |S|^{1/4}}{\lambda} + \frac{J L_\cC^2}{\lambda^2} \right) \ln(dJ|\cA||\cN||\cS|T/\delta) \right).
    \]
\end{theorem}
\newcommand{\proofTheoremCCV} {
    Fix any $j \in [J]$. We will bound the cumulative violation against the $j$-th constraint, i.e., $\sum_{t \in S} c_j(a_t, y_t)$. The final result follows by taking the maximum over all $j \in [J]$.

    We use the triangle inequality to decompose the sum:
    \begin{align*}
        \sum_{t \in S} c_j(a_t,y_t) &\le \left| \sum_{t \in S} c_j(a_t,y_t) - \sum_{t \in S} c_j(a_t,p_t) \right| + \sum_{t \in S} c_j(a_t,p_t)
    \end{align*}
    We bound each of the two terms on the right-hand side separately.

    For the first term, we partition the sum based on the action played and use the linearity of the constraint functions to derive that:
    \begin{align*}
        \left| \sum_{t \in S} c_j(a_t,y_t) - \sum_{t \in S} c_j(a_t,p_t) \right| &= \left| \sum_{a \in \cA} \sum_{t \in S} c_j(a,y_t) \1[a_t=a] - \sum_{a \in \cA}\sum_{t \in S} c_j(a,p_t) \1[a_t=a] \right| \\
        &= \left| \sum_{a \in \cA} c_j\left( a, \sum_{t \in S} y_t \1[a_t=a] \right) - \sum_{a \in \cA} c_j\left( a, \sum_{t \in S} p_t \1[a_t=a] \right) \right| \\
        &\le \sum_{a \in \cA} \left| c_j\left( a, \sum_{t \in S} y_t \1[a_t=a] \right) - c_j\left( a, \sum_{t \in S} p_t \1[a_t=a] \right) \right| \\
        &\le \sum_{a \in \cA} L_\cC \left| \sum_{t \in S} y_t \1[a_t=a] - \sum_{t \in S} p_t \1[a_t=a] \right| \\
        &\le \sum_{a \in \cA} L_\cC \alpha(T^{u,\bc,S}(a))
    \end{align*}
    where the first inequality follows from the triangle inequality, the second inequality follows from $L_\cC$-Lipschitzness of $c_j$, the third inequality follows from $(\cN,\cS,\alpha)$-decision calibration. By concavity of $\alpha$ and the fact that $\sum_{a \in \cA} T^{u,\bc,S}(a) = \sum_{a \in \cA} \sum_{t=1}^T \1\left[t \in S, \CBR^{u,\bc}(p_t) = a\right] = |S|$, this expression is at most:
    \begin{align*}
        L_\cC |\cA| \alpha(|S|/|\cA|).
    \end{align*}
    
    For the second term, $\sum_{t \in S} c_j(a_t,p_t)$, we decompose the sum based on whether the predicted feasible set $\widehat\cA_t^{\bc,\fea}$ is empty on round $t$.
    By the agent's decision rule, on any round where $\widehat{\mathcal{A}}_t^{\mathbf{c},\text{fea}} \neq \emptyset$, the chosen action $a_t$ satisfies $c_j(a_t,p_t) \le 0$. On rounds where the set is empty, the violation is at most 1. The sum is therefore bounded by the number of ``empty set'' rounds:
    \begin{align*}
        \sum_{t \in S} c_j(a_t,p_t) &= \sum_{t \in S: \widehat\cA_t^{\bc,\fea} \neq \emptyset} c_j(a_t,p_t) + \sum_{t \in S: \widehat\cA_t^{\bc,\fea} = \emptyset} c_j(a_t,p_t) \\
        &\le \sum_{t \in S: \widehat\cA_t^{\bc,\fea} \neq \emptyset} 0 + \sum_{t \in S: \widehat\cA_t^{\bc,\fea} = \emptyset} 1 \\
        &= \left|\left\{ t \in S: \widehat\cA_t^{\bc,\fea} = \emptyset \right\}\right|
    \end{align*}
    By Lemma \ref{lem:infeasibility-calibration}, this expression is at most:
    \begin{align*}
        J f_\beta(L_\cC/\lambda)
    \end{align*}
    
    Combining the two bounds gives the result stated in the theorem.
}
\ifarxiv
\begin{proof}
    \proofTheoremCCV
\end{proof}
\else
\fi

We note that this guarantee holds simultaneously for all choice of $\lambda$ (as long as the corresponding benchmark class $\cA_S^{\bc,\lambda}$ is non-empty). By setting $\lambda = |S|^{-1/4}$ for each subsequence $S \in \cS$, we arrive at the following concrete bound for the cumulative constraint violation.
\begin{corollary} \label{cor:CCV}
    Let $\cS$ be a collection of subsequences. Let $\cN$ be a set of agents, where each agent is equipped with a utility function $u: \cA \times \cY \to [0,1]$ and $J$ constraint functions $\{c_j: \cA \times \cY \to [-1,1]\}_{j \in [J]}$. Suppose each agent plays constrained best responses to $p_t$ to compete with actions from the benchmark class $\cA_S^{\bc,|S|^{-1/4}}$ over each subsequence $S \in \cS$. The sequence of predictions $p_1,\ldots,p_T$ produced by \textsc{Decision-Infeasibility-Calibration} guarantees that with probability at least $1-\delta$, the cumulative constraint violation of any agent over any subsequence $S \in \cS$ is bounded by:
    \begin{align*}
        \CCV(S) \le O\left( \left( L_\cC |\cA| |S|^{1/4} + L_\cC \sqrt{|\cA||S|} + J (L_\cC+L_\cC^2) \sqrt{|S|} \right) \ln(dJ|\cA||\cN||\cS|T/\delta) \right).
    \end{align*}
\end{corollary}

\begin{remark}
    Our analysis can be extended to the standard benchmark with zero margin ($\lambda=0$). This requires relaxing the agent's decision rule to allow for a small tolerance in predicted feasibility (i.e., choosing from actions where $c_j(a, p_t) \le \eta$). A similar analysis reveals a trade-off, yielding a bound of roughly $\tilde{O}(1/\eta^2 + \eta|S|)$. Optimizing $\eta$ results in cumulative constraint violation and regret bounds of $\tilde{O}(T^{2/3})$. We defer the full results to Appendix \ref{app:zero-margin}.
\end{remark}

\subsection{Bounding the Regret}

Next we show that decision calibration and infeasibility calibration imply no constrained swap regret, and hence no constrained external regret. At a high level, decision calibrated predictions allow agents to accurately assess (on average) the utilities of their chosen actions and the counterfactual actions produced by any action modification rule. Because agents choose optimally based on these accurate utility estimates, their decisions are  competitive against any alternative action as long as the alternative action is also predicted to be feasible. Infeasibility calibrated predictions guarantee that the alternative action must be predicted to be feasible at almost every round. 
\begin{theorem} \label{thm:swap-regret}
    Let $\cS$ be a collection of subsequences. Let $\lambda: \bN \to (0,\infty)$ be a margin function. Let $\cN$ be a set of agents, where each agent is equipped with a utility function $u: \cA \times \cY \to [0,1]$ and $J$ constraint functions $\{c_j: \cA \times \cY \to [-1,1]\}_{j \in [J]}$. Suppose each agent plays constrained best responses to $p_t$ to compete with actions from the benchmark class $\cA_S^{\bc,\lambda}$ over each subsequence $S \in \cS$. If the sequence of predictions $p_1,\ldots,p_T$ is $(\cN,\cS,\alpha)$-decision calibrated and $(\cN,\cS,\beta)$-infeasibility calibrated, then the constrained swap regret of any agent over any subsequence $S \in \cS$ is bounded by:
    \begin{align*}
        \Reg_\swap(u,\bc,\lambda,S) \le 2 L_\cU |\cA| \alpha(|S|/|\cA|) + J |\cA| f_\beta(L_\cC/\lambda).
    \end{align*}
    In particular, plugging in the guarantee from  Theorem \ref{thm:unbiased-algorithm} yields the following bound, which holds with probability at least $1-\delta$:
    \begin{align*}
        \Reg_\swap(u,\bc,\lambda,S) \le O\left( \left( L_\cU |\cA| |S|^{1/4} + L_\cU \sqrt{|\cA||\cS|} + \frac{J L_\cC |\cA| |S|^{1/4}}{\lambda} + \frac{J L_\cC^2 |\cA|}{\lambda^2} \right) \ln(dJ|\cA||\cN||\cS|T/\delta) \right).
    \end{align*}
\end{theorem}

\newcommand{\proofTheoremSwapRegretPartOne}{
To prove the theorem, first note that for any subsequence $S \in \cS$ and action modification rule $\phi: \cA \to \cA_{S}^\bc$, we can decompose the regret against $\phi$ into three parts as:
\begin{align*}
    \MoveEqLeft \sum_{t \in S} \left( u(\phi(a_t),y_t) - u(a_t,y_t) \right) \\
    &= \sum_{t \in S} \left( u(\phi(a_t),y_t) - u(\phi(a_t),p_t) \right) + \sum_{t \in S} \left( u(\phi(a_t),p_t) - u(a_t,p_t) \right) + \sum_{t \in S} \left( u(a_t,p_t) - u(a_t,y_t) \right)
\end{align*}

We first bound the first and third part, i.e., the difference in utility under our predictions $p_t$ and the outcomes $y_t$ for both the chosen actions and the swapped-in actions. We show this in the next two lemmas using decision calibration. 
\begin{lemma} \label{lem:decision-calibration}
    If the sequence of predictions $p_1,\ldots,p_T$ is $(\cN,\cS,\alpha)$-decision calibrated, then for any $(u,\bc) \in \cN$ and $S \in \cS$:
    \begin{align*}
        \left| \sum_{t \in S} (u(a_t,p_t) - u(a_t,y_t)) \right| \le L_\cU |\cA| \alpha(|S|/|\cA|).
    \end{align*}
\end{lemma}
\begin{proof}
    Using the linearity of $u$, we can write:
    \begin{align*}
        \left| \sum_{t \in S} (u(a_t,p_t) - u(a_t,y_t)) \right| &= \left| \sum_{a \in \cA}\sum_{t=1}^T \1[t \in S,a_t=a] (u(a,p_t) - u(a,y_t)) \right| \\
        &= \left| \sum_{a \in \cA} \left( u\left(a, \sum_{t=1}^T \1[t \in S,a_t=a] p_t \right) - u\left(a, \sum_{t=1}^T \1[t \in S,a_t=a] y_t\right) \right)\right| \\
        &\le \sum_{a \in \cA} \left| u\left(a, \sum_{t=1}^T \1[t \in S,a_t=a] p_t \right) - u\left(a, \sum_{t=1}^T \1[t \in S,a_t=a] y_t\right) \right| \\
        &\le \sum_{a \in \cA} L_\cU \left\| \sum_{t=1}^T \1[t \in S,a_t=a] (p_t-y_t) \right\|_\infty \\
        &\le L_\cU \sum_{a \in \cA} \alpha(T^{u,\bc,S}(a)).
    \end{align*}
    where the first inequality follows from the triangle inequality, the second inequality follows from $L_\cU$-Lipschitzness of $u$, and the third inequality follows from $(\cN,\cS,\alpha)$-decision calibration. By concavity of $\alpha$ and the fact that $\sum_{a \in \cA} T^{u,\bc,S}(a) = \sum_{a \in \cA} \sum_{t=1}^T \1[t \in S, a_t=a] = |S|$, this expression is at most:
    \begin{align*}
        L_\cU |\cA| \alpha(|S|/|\cA|).
    \end{align*}
\end{proof}

\begin{lemma} \label{lem:decision-calibration-swap}
    If the sequence of predictions $p_1,\ldots,p_T$ is $(\cN,\cS,\alpha)$-decision calibrated, then for any $(u,\bc) \in \cN$ and $S \in \cS$:
    \begin{align*}
        \left| \sum_{t \in S} (u(\phi(a_t),p_t) - u(\phi(a_t),y_t)) \right| \le L_\cU |\cA| \alpha(|S|/|\cA|).
    \end{align*}
\end{lemma}
\begin{proof}
    Using the linearity of $u$, we can write:
    \begin{align*}
        \left| \sum_{t \in S} (u(\phi(a_t),p_t) - u(\phi(a_t),y_t)) \right| &= \left| \sum_{a \in \cA}\sum_{t=1}^T \1[t \in S,a_t=a] (u(\phi(a),p_t) - u(\phi(a),y_t)) \right| \\
        &= \left| \sum_{a \in \cA} \left( u\left(\phi(a), \sum_{t=1}^T \1[t \in S,a_t=a] p_t \right) - u\left(\phi(a), \sum_{t=1}^T \1[t \in S,a_t=a] y_t\right) \right)\right| \\
        &\le \sum_{a \in \cA} \left| u\left(\phi(a), \sum_{t=1}^T \1[t \in S,a_t=a] p_t \right) - u\left(\phi(a), \sum_{t=1}^T \1[t \in S,a_t=a] y_t\right) \right| \\
        &\le \sum_{a \in \cA} L_\cU \left\| \sum_{t=1}^T \1[t \in S,a_t=a] (p_t-y_t) \right\|_\infty \\
        &\le L_\cU \sum_{a \in \cA} \alpha(T^{u,\bc,S}(a)) \\
        &\le L_\cU |\cA| \alpha(|S|/|\cA|)
    \end{align*}
    where the first inequality follows from the triangle inequality, the second inequality follows from $L_\cU$-Lipschitzness of $u$, the third inequality follows from $(\cN,\cS,\alpha)$-decision calibration, and the fourth inequality follows from concavity of $\alpha$.
\end{proof}

Regarding the second part in our decomposition of the regret against $\phi$, we further decompose it into two components based on whether the swapped-in action is predicted to be feasible:
\begin{align*}
    \sum_{t \in S} \left( u(\phi(a_t),p_t) - u(a_t,p_t) \right) &= \sum_{t \in S: \phi(a_t) \in \widehat\cA_t^{\bc,\fea}} \left( u(\phi(a_t),p_t) - u(a_t,p_t) \right) \\ &\quad + \sum_{t \in S: \phi(a_t) \in \widehat\cA_t^{\bc,\infea}} \left( u(\phi(a_t),p_t) - u(a_t,p_t) \right)
\end{align*}

The first component is non-positive. This is because on these rounds, $a_t$ is chosen to maximize predicted utility over the predicted feasible set, which includes $\phi(a_t)$, so $u(\phi(a_t),p_t) \le u(a_t,p_t)$.

For the second component, the utility difference is at most 1, so the sum is bounded by the number of rounds within $S$ on which the swapped-in action $\phi(a_t)$ is predicted to be infeasible. We can bound this count by the sum of the number of rounds where each single benchmark action $a \in \cA_S^{\bc,\lambda}$ is predicted to be infeasible.
\begin{align*}
    \sum_{t \in S: \phi(a_t) \in \widehat\cA_t^{\bc,\infea}} \left( u(\phi(a_t),p_t) - u(a_t,p_t) \right) &\le \sum_{t \in S} \1\left[\phi(a_t) \in \widehat\cA_t^{\bc,\infea}\right] \\
    &= \sum_{a \in \cA_S^{\bc,\lambda}} \sum_{t \in S} \1\left[\phi(a_t) = a, a \in \widehat\cA_t^{\bc,\infea}\right] \\
    &\le \sum_{a \in \cA_S^{\bc,\lambda}} \sum_{t \in S} \1\left[a \in \widehat\cA_t^{\bc,\infea}\right] \\
    &= \sum_{a \in \cA_S^{\bc,\lambda}} \sum_{t \in S} \1\left[\exists j \in [J], c_j(a,p_t)>0\right] \\
    &\le \sum_{a \in \cA_S^{\bc,\lambda}} \sum_{t \in S} \sum_{j=1}^J \1\left[c_j(a,p_t)>0\right] \\
    &= \sum_{a \in \cA_S^{\bc,\lambda}} \sum_{j=1}^J T^{c_j,S,\infea}(a)
\end{align*}

By Lemma \ref{lem:infeasibility-calibration}, this expression is at most:
\begin{align*}
    J |\cA| f_\beta(L_\cC/\lambda)
\end{align*}
}

\newcommand{\proofTheoremSwapRegretPartTwo}{
    For any subsequence $S \in \cS$ and action modification rule $\phi: \cA \to \cA_{S}^\bc$, we can decompose the regret against $\phi$ into three parts as:
    \begin{align*}
        \MoveEqLeft \sum_{t \in S} \left( u(\phi(a_t),y_t) - u(a_t,y_t) \right) \\
        &= \sum_{t \in S} \left( u(\phi(a_t),y_t) - u(\phi(a_t),p_t) \right) + \sum_{t \in S} \left( u(\phi(a_t),p_t) - u(a_t,p_t) \right) + \sum_{t \in S} \left( u(a_t,p_t) - u(a_t,y_t) \right)
    \end{align*}
    By Lemmas \ref{lem:decision-calibration} and \ref{lem:decision-calibration-swap}, the first and third part are both bounded by $L_\cU |\cA| \alpha(|S|/|\cA|)$.
    
    As shown in our preceding analysis, the second part is bounded by $J |\cA| f_\beta(L_\cC/\lambda)$.
    
    Combining the bounds for all three parts, we arrive at the final inequality.
}

\ifarxiv
\proofTheoremSwapRegretPartOne
We can now complete the proof of Theorem \ref{thm:swap-regret}.
\begin{proof}
    \proofTheoremSwapRegretPartTwo
\end{proof}
\else
\fi

By setting $\lambda = |S|^{-1/4}$ for each subsequence $S \in \cS$, we arrive at the following concrete bound for the constrained swap regret.
\begin{corollary} \label{cor:swap-regret}
    Let $\cS$ be a collection of subsequences. Let $\cN$ be a set of agents, where each agent is equipped with a utility function $u: \cA \times \cY \to [0,1]$ and $J$ constraint functions $\{c_j: \cA \times \cY \to [-1,1]\}_{j \in [J]}$. Suppose each agent plays constrained best responses to $p_t$ to choose action $a_t$. The sequence of predictions $p_1,\ldots,p_T$ produced by \textsc{Decision-Infeasibility-Calibration} guarantees that with probability at least $1-\delta$, the constrained swap regret against the benchmark class $\cA_S^{\bc,|S|^{-1/4}}$ of any agent over any subsequence $S \in \cS$ is bounded by:
    \begin{align*}
        \Reg_\swap(u,\bc,|S|^{-1/4},S) 
        &\le O\left( \left( L_\cU |\cA| |S|^{1/4} + L_\cU \sqrt{|\cA||S|} + J (L_\cC+L_\cC^2) |\cA| \sqrt{|S|} \right) \ln(dJ|\cA||\cN||\cS|T/\delta) \right).
    \end{align*}
\end{corollary}

\begin{remark}
    One could alternatively set a fixed margin of $\lambda = T^{-1/4}$. This choice creates a larger and more competitive benchmark class. The resulting regret bound would then be $\tilde{O}(\sqrt{T})$.
\end{remark}

A powerful implication is that we achieve low constrained swap adaptive regret by instantiating our framework with the collection of all contiguous intervals, $\cS = \{[t_1,t_2]: 1 \le t_1 \le t_2 \le T\}$.
\begin{corollary} \label{cor:swap-adaptive-regret}
    Let $\cN$ be a set of agents, where each agent is equipped with a utility function $u: \cA \times \cY \to [0,1]$ and $J$ constraint functions $\{c_j: \cA \times \cY \to [-1,1]\}_{j \in [J]}$. Suppose each agent plays constrained best responses to $p_t$ to choose action $a_t$. 
    Let $\lambda(|S|) = |S|^{-1/4}$ be the margin function. 
    The sequence of predictions $p_1,\ldots,p_T$ produced by \textsc{Decision-Infeasibility-Calibration} guarantees that with probability at least $1-\delta$, the constrained swap adaptive regret of any agent is bounded by:
    \begin{align*}
        \Reg_{\swap-\adapt}(u,\bc,\lambda) \le O\left( \left( L_\cU |\cA| T^{1/4} + L_\cU \sqrt{|\cA|T} + J (L_\cC+L_\cC^2) |\cA| \sqrt{T} \right) \ln(dJ|\cA||\cN|T/\delta) \right).
    \end{align*}
\end{corollary}

A dynamic benchmark with $\Delta$ changes partitions the entire time horizon into $\Delta+1$ intervals. By summing our per-subsequence regret bound over this specific partition, we obtain the following dynamic regret guarantee.
\begin{corollary} \label{cor:swap-dynamic-regret}
    Let $\cN$ be a set of agents, where each agent is equipped with a utility function $u: \cA \times \cY \to [0,1]$ and $J$ constraint functions $\{c_j: \cA \times \cY \to [-1,1]\}_{j \in [J]}$. Suppose each agent plays constrained best responses to $p_t$ to choose action $a_t$. The sequence of predictions $p_1,\ldots,p_T$ produced by \textsc{Decision-Infeasibility-Calibration} guarantees that with probability at least $1-\delta$, the constrained swap dynamic regret of any agent against any piecewise feasible sequence of action modification rule $\vec\phi \in (\cA^\cA)^T$ is bounded by:
    \begin{align*}
        \Reg_{\swap-\dyn}(u,\vec\phi) \le O\left( \left( L_\cU |\cA| T^\frac{1}{4} \Delta(\vec\phi)^\frac{3}{4} + L_\cU \sqrt{|\cA| T \Delta(\vec\phi)} + J (L_\cC+L_\cC^2) |\cA| \sqrt{T \Delta(\vec\phi)} \right) \ln(dJ|\cA||\cN||\cS|T/\delta) \right).
    \end{align*}
    A benchmark sequence $\vec\phi$ is piecewise feasible if on each interval of constancy $I_k$, the corresponding rule $\psi_k$ maps to the set of actions that are feasible over that entire interval, i.e., $\psi_k: \cA \to \cA_{I_k}^{\bc,|I_k|^{-1/4}}$.
\end{corollary}

\newcommand{\proofSwapDynamicRegret}{
    Suppose the sequence $\vec\phi$ have change points that partion $[1,T]$ into intervals $I_0,\ldots,I_{\Delta(\vec\phi)}$. The action modification rule on each interval $I_k$ is the constant rule $\psi_k: \cA \to \cA_{I_k}^{\bc,\lambda}$.

    By Theorem \ref{thm:swap-regret}, for any interval $I_k$, the regret against $\psi_k$ over $I_k$ is bounded by:
    \begin{align*}
        \MoveEqLeft \sum_{t=\min\{I_k\}}^{\max\{I_k\}} (u(\psi_k(a_t),y_t) - u(a_t,y_t)) \\
        &\le O\left( \left( L_\cU |\cA| |I_k|^{1/4} + L_\cU \sqrt{|\cA||I_k|} + \frac{J L_\cC |\cA| |I_k|^{1/4}}{\lambda} + \frac{J L_\cC^2 |\cA|}{\lambda^2} \right) \ln(dJ|\cA||\cN||\cS|T/\delta) \right)
    \end{align*}

    Summing over all the intervals, we have:
    \begin{align*}
        \MoveEqLeft \sum_{t=1}^{T} (u(\psi_k(a_t),y_t) - u(a_t,y_t)) \\
        &= \sum_{k=0}^{\Delta(\vec\phi)} \sum_{t=\min\{I_k\}}^{\max\{I_k\}} (u(\psi_k(a_t),y_t) - u(a_t,y_t)) \\
        &\le \sum_{k=0}^{\Delta(\vec\phi)} O\left( \left( L_\cU |\cA| |I_k|^{1/4} + L_\cU \sqrt{|\cA||I_k|} + \frac{J L_\cC |\cA| |I_k|^{1/4}}{\lambda} + \frac{J L_\cC^2 |\cA|}{\lambda^2} \right) \ln(dJ|\cA||\cN||\cS|T/\delta) \right)
    \end{align*}

    We note that $\lambda$ can be set independently for each interval $I_k$. By setting $\lambda = |I_k|^{-1/4}$, we arrive at the following bound:
    \begin{align*}
        \sum_{k=0}^{\Delta(\vec\phi)} O\left( \left( L_\cU |\cA| |I_k|^{1/4} + L_\cU \sqrt{|\cA||I_k|} + J (L_\cC+L_\cC^2) |\cA| \sqrt{|I_k|} \right) \ln(dJ|\cA||\cN||\cS|T/\delta) \right)
    \end{align*}

    By concavity of the functions $f_1(x)=x^{1/4}$ and $f_2(x)=x^{1/2}$ and the fact that $\sum_{k=0}^{\Delta(\vec\phi)} |I_k| = T$, this expression is at most:
    \begin{align*}
        O\left( \left( L_\cU |\cA| T^{1/4} \Delta(\vec\phi)^{3/4} + L_\cU \sqrt{|\cA| T \Delta(\vec\phi)} + J (L_\cC+L_\cC^2) |\cA| \sqrt{T \Delta(\vec\phi)} \right) \ln(dJ|\cA||\cN||\cS|T/\delta) \right)
    \end{align*}
}

\ifarxiv
\begin{proof}
    \proofSwapDynamicRegret
\end{proof}
\else
\fi

Since external regret is a special case of swap regret, our bounds apply directly to the external versions of these guarantees as well.

\bibliographystyle{ACM-Reference-Format}
\bibliography{references}

\appendix

\section{Additional Related Work} \label{app:related-work}
The study of online learning with long-term constraints was initiated by \cite{MannorTY09}, who established a foundational impossibility result. They demonstrated that learners cannot achieve sublinear external regret when benchmarked against the set of actions that satisfy the constraints marginally, i.e., on average over the entire time horizon. In response to this limitation, subsequent work pivoted to a more stringent benchmark: actions that fulfill the constraints on every single round (e.g., \cite{Sun17safety,guo2022online,anderson2022lazy,yi2023distributed,sinha2024optimal,lekeufack2024optimistic}). These papers generally study the problem within the online convex optimization framework, rather than for a finite number of actions. Within this body of work, \cite{sinha2024optimal} achieved the fastest known rates without additional structural assumptions, with both regret and cumulative constraint violation scaling as $\tilde{O}(\sqrt{T})$. For the experts setting, \cite{lekeufack2024optimistic} proposed an algorithm where both metrics scale as $\tilde{O}(\sqrt{T\ln(d)})$, with $d$ representing the number of experts and regret measured against the set of expert probability distributions that satisfy the constraints in expectation.

This literature has also studied dynamic regret benchmarks (\cite{ChenLG17,chen2018,chen2018bandit,CaoL19,vaze2022dynamic,liu2022simultaneously,lekeufack2024optimistic}), where regret is measured against a slowly changing sequence of actions, one for every round in the interaction. We introduce the stronger notion of dynamic swap regret. 

Some papers in this literature (e.g. \citet{neelyYu2017,ChenLG17,CaoL19,YuNeely2020,Castiglioni22}) assume Slater's condition holds --- the existence of a benchmark action that satisfies all of the constraints with constant margin (or a sequence of strongly feasible actions, in the dynamic benchmark case), which is an assumption we sometimes make in this work as well. Note that we also provide bounds that hold without this assumption.

Omniprediction was  introduced by \cite{gopalan2021omnipredictors} within the batch (distributional) learning setting. In its standard formulation, omniprediction considers a binary label space $\cY = \{0,1\}$, a continuous action space of real-valued predictions $\cA = [0,1]$, and decision-makers who optimize a loss function $\ell:[0,1] \times \{0,1\} \rightarrow \mathbb{R}$, such as squared or absolute error. A key initial finding was that multicalibration (c.f. \cite{hebert2018multicalibration}) is a sufficient condition for achieving omniprediction. Subsequent work \citep{gopalan2022loss,gopalan2023swap} studied both weaker and stronger notions of calibration and the corresponding omniprediction guarantees they provide.

Omniprediction was later extended to the online adversarial setting by \cite{garg2024oracle}, who provided oracle-efficient algorithms, with subsequent work establishing optimal regret bounds \citep{okoroafor2025near,dwork2024fairness}. While the bulk of this literature focuses on binary outcomes, a notable exception is \cite{gopalan2024omnipredictors}, which investigates vector-valued outcomes for decision-makers with convex loss functions. The problem of omniprediction with constraints has been explored in \citep{globus2023multicalibrated,hu2023omnipredictors}. Consistent with the broader literature, these works operate in the batch setting with binary outcomes and continuous actions, focusing on constraints motivated by group fairness in machine learning classification tasks.

A parallel line of work has emerged on making sequential predictions for downstream agents \cite{kleinberg2023u,noarov2023highdimensional,roth2024forecasting,hu2024predict}. This literature considers adversarially chosen outcomes from a space $\cY$ and downstream agents who possess arbitrary discrete action spaces and seek to optimize their own losses. Several of these papers move beyond the binary setting \cite{noarov2023highdimensional,roth2024forecasting}, addressing scenarios---similar to our own---where loss functions are linear in the high-dimensional state $y \in [0,1]^d$ to be predicted. The primary objective is to provide worst-case, diminishing regret guarantees for all such agents. Motivated by competitive environments, a subset of this work \cite{noarov2023highdimensional,roth2024forecasting,hu2024predict} provides guarantees for diminishing swap regret at near-optimal rates.

Recognizing that traditional calibration is unobtainable at $O(\sqrt{T})$ rates in the online adversarial setting \citep{qiao2021stronger,dagan2024improved}, this literature has employed alternative techniques. Methods such as ``U-calibration'' \citep{kleinberg2023u} and extensions of decision calibration \cite{zhao2021calibrating,noarov2023highdimensional,roth2024forecasting} circumvent these known lower bounds while still being powerful enough to ensure downstream agents incur no (swap) regret. A unifying perspective on these two research areas is offered by \cite{lu2025sample}. Our own model and techniques are primarily derived from this ``prediction for downstream regret'' literature, adopting its features of arbitrary action spaces, high-dimensional outcomes, and swap regret guarantees in an online adversarial context.

Finally, our work is a direct follow up to \cite{BLR25} who introduced the problem of ``omniprediction with long term constraints''. Like \cite{BLR25}, we give an algorithm that can guarantee an arbitrary collection of decision makers diminishing regret bounds on an arbitrarily specified collection of subsequences. Our main technical contribution is an exponentially improved dependence (in both the regret and constraint violation terms) on the number of such subsequences --- the algorithm of \cite{BLR25} depends \emph{linearly} on this, whereas our bounds depend only \emph{logarithmically} on this parameter. This is crucial in order to give dynamic regret bounds, which fall out of guaranteeing diminishing (swap) regret on all contiguious subsequences, of which there are $\Omega(T^2)$. This would yield trivial bounds using the algorithm of \cite{BLR25}, whereas it costs only an additional logarithmic term in our regret bounds relative to what we could guarantee for a static benchmark.

\section{Analysis for the Zero-Margin Benchmark} \label{app:zero-margin}

\subsection{A Decision Rule with Feasibility Tolerance}

In this section, we obtain guarantees for the standard benchmark $\cA_S^{\bc,0}$ with zero margin $\lambda=0$. For this purpose, we modify the agent's decision rule. We introduce a relaxed rule that allows for a small, positive tolerance in predicted feasibility $\eta>0$. An action is now considered feasible by the agent if its predicted constraint violation is less than or equal to this tolerance.

Formally, we say that an action $a$ is predicted to $\eta$-violate the $j$-th constraint if $c_j(a,p_t) > \eta$.
An action $a$ is predicted to be $\eta$-infeasible if $a$ is predicted to $\eta$-violate any constraint, and is predicted to be $\eta$-feasible otherwise.
This leads to a relaxed version of the constrained best response.

\begin{definition}[$\eta$-Constrained Best Response] 
    Fix a utility $u: \cA \times \cY \to [0,1]$, $J$ constraints $\{c_j: \cA \times \cY \to [-1,1]\}_{j \in [J]}$, a prediction $p \in \cY$, and a tolerance $\eta>0$. The $\eta$-constrained best response to $p$ according to $u$ and $\{c_j\}_{j \in [J]}$, denoted as $\CBR_\eta^{u,\bc}(p)$, is the solution to the constrained optimization problem:
    \begin{align*}
        \underset{a \in \cA}{\text{maximize}} &\quad u(a,p_t) \\
        \text{subject to} &\quad c_j(a,p_t) \le \eta \; \text{ for every } j \in [J]
    \end{align*}
\end{definition}

The agent can obtain $\CBR_\eta^{u,\bc}(p_t)$ in two steps:
\begin{enumerate}[(1)]
    \item 
    The agent first discards actions that are $\eta$-infeasible according to the prediction.

    For each constraint $j \in [J]$, the set of actions predicted to $\eta$-violate that constraint is denoted as:
    $$\widehat{\cA}_t^{c_j,\eta-\infea} = \left\{ a \in \cA: c_j(a,p_t) > \eta \right\}.$$
    
    The agent discards actions that are predicted to $\eta$-violate any of the $J$ constraints, i.e.,
    $$\widehat{\cA}_t^{\bc,\eta-\infea} = \cup_{j \in [J]} \widehat{\cA}_t^{c_j,\eta-\infea} = \left\{ a \in \cA: \exists j \in [J], c_j(a,p_t) > \eta \right\}.$$
    
    The retained actions that are predicted to be $\eta$-feasible are denoted as $$\widehat{\cA}_t^{\bc,\eta-\fea} = \left\{ a \in \cA: \forall j \in [J], c_j(a,p_t) \le \eta \right\}.$$
    \item
    The agent then chooses an action from the retained action set $\widehat{\cA}_t^{\bc,\eta-\fea}$ that maximizes the utility function according to the prediction, i.e.,
    \begin{align*}
        a_t = \argmax_{a \in \widehat{\cA}_t^{\bc,\eta-\fea}} u(a,p_t).
    \end{align*}
\end{enumerate}

If none of the actions are predicted to be feasible, i.e., $\widehat\cA_t^{\bc,\eta-\fea} = \emptyset$, the agent can choose any arbitrary action from $\cA$. Our predictions ensure this special case rarely occurs, and hence its influence on the cumulative constraint violation and regret is negligible.

\subsection{Conditionally Unbiased Predictions}

As in our main analysis, the guarantees for the relaxed decision rule rely on the predictions being conditionally unbiased. The definitions for decision calibration and infeasibility calibration are analogous to those in the main text, modified to account for the feasibility tolerance $\eta$.

\begin{definition}[$(\cN,\cS,\eta,\alpha)$-Decision Calibration]
    Let $\cS$ be a collection of subsequences. Let $\cN$ be a set of agents, where each agent is equipped with a utility function $u: \cA \times \cY \to [0,1]$ and $J$ constraint functions $\{c_j: \cA \times \cY \to [-1,1]\}_{j \in [J]}$. 
    Let $\eta>0$ be the feasibility tolerance.
    We say that a sequence of predictions $p_1,\ldots,p_T$ is $(\cN,\cS,\eta,\alpha)$-decision calibrated with respect to a sequence of outcomes $y_1,\ldots,y_T$ if for every $S \in \cS$, $a \in \cA$, and $(u,\bc) \in \cN$:
    \[
        \left\| \sum_{t=1}^T \1\left[t \in S, \CBR_\eta^{u,\bc}(p_t) = a \right] (p_t - y_t) \right\|_\infty \leq \alpha(T^{u,\bc,S,\eta}(a))
    \]
    where $T^{u,\bc,S,\eta}(a) = \sum_{t=1}^T \1\left[t \in S, \CBR_\eta^{u,\bc}(p_t) = a\right]$.
\end{definition}

\begin{definition}[$(\cN,\cS,\eta,\beta)$-Infeasibility Calibration]
    Let $\cS$ be a collection of subsequences. Let $\cN$ be a set of agents, where each agent is equipped with a utility function $u: \cA \times \cY \to [0,1]$ and $J$ constraint functions $\{c_j: \cA \times \cY \to [-1,1]\}_{j \in [J]}$. 
    Let $\eta>0$ be the feasibility tolerance.
    We say that a sequence of predictions $p_1,\ldots,p_T$ is $(\cN,\cS,\eta,\beta)$-infeasibility calibrated with respect to a sequence of outcomes $y_1,\ldots,y_T$ if for every $S \in \cS$, $a \in \cA$, $(u,\bc) \in \cN$, and $j \in [J]$:
    \[
        \left\| \sum_{t=1}^T \1\left[t \in S, a \in \widehat\cA_t^{c_j,\eta-\infea}\right] (p_t - y_t) \right\|_\infty \leq \beta(T^{c_j,S,\eta-\infea}(a))
    \]
    where $T^{c_j,S,\eta-\infea}(a) = \sum_{t=1}^T \1\left[t \in S, a \in \widehat\cA_t^{c_j,\eta-\infea}\right]$.
\end{definition}

We will again instantiate \textsc{Unbiased-Prediction} to make predictions that simultaneously achieve decision calibration and infeasibility calibration; we will refer to this instantiation as \textsc{Decision-Infeasibility-Calibration-Relaxed}. 
Our guarantees will inherit from the guarantees of \textsc{Unbiased-Prediction}. 

\begin{theorem} \label{thm:unbiased-algorithm-zero-margin}
    Let $\cS$ be a collection of subsequences. Let $\cN$ be a set of agents, where each agent is equipped with a utility function $u: \cA \times \cY \to [0,1]$ and $J$ constraint functions $\{c_j: \cA \times \cY \to [-1,1]\}_{j \in [J]}$.
    Let $\eta>0$ be the feasibility tolerance.
    There is an instantiation of \textsc{Unbiased-Prediction} \citep{noarov2023highdimensional}
    ---which we call \textsc{Decision-Infeasibility-Calibration-Relaxed}---
    producing predictions $p_1,...,p_T \in \cY$ satisfying that for any sequence of outcomes $y_1,...,y_T \in \cY$, 
    with probability at least $1-\delta$, for any $(u,\bc) \in \cN$, $j \in [J]$, $a \in \cA$, and $S \in \cS$:
        \begin{align*}
            & \left\| \sum_{t=1}^T \1\left[t \in S, \CBR_\eta^{u,\bc}(p_t) = a \right] (p_t - y_t) \right\|_\infty \leq O\left( \ln\frac{dJ|\cA||\cN||\cS|T}{\delta} \cdot |S|^{1/4} + \sqrt{\ln\frac{dJ|\cA||\cN||\cS|T}{\delta} \cdot T^{u,\bc,S,\eta}(a)} \right) \\
            & \left\| \sum_{t=1}^T \1\left[t \in S, a \in \widehat\cA_t^{c_j,\eta-\infea}\right] (p_t - y_t) \right\|_\infty \leq O\left( \ln\frac{dJ|\cA||\cN||\cS|T}{\delta} \cdot |S|^{1/4} + \sqrt{\ln\frac{dJ|\cA||\cN||\cS|T}{\delta} \cdot T^{c_j,S,\eta-\infea}(a)} \right)
        \end{align*}
\end{theorem}

\subsection{Theoretical Guarantees}

Similarly to Lemma \ref{lem:infeasibility-calibration}, we use the infeasibility calibration guarantee to show that for any benchmark action $a \in \cA_S^{\bc,0}$, the number of rounds where it is incorrectly predicted to $\eta$-violate a specific constraint $c_j$ is small. As a result, the number of rounds on which no action is predicted to be $\eta$-feasible is small.
\begin{lemma} \label{lem:infeasibility-calibration-zero-margin}
    If the sequence of predictions $p_1,\ldots,p_T$ is $(\cN,\cS,\eta,\beta)$-infeasibility calibrated, then for any agent $(u,\bc) \in \cN$, subsequence $S \in \cS$, benchmark action $a \in \cA_S^{\bc,0}$, and constraint $j \in [J]$, the number of rounds $T^{c_j,S,\eta-\infea}(a)$ within $S$ on which $a$ is predicted to $\eta$-violate the $j$-th constraint is bounded by:
    \begin{align*}
        T^{c_j,S,\eta-\infea}(a) \le f_\beta(L_\cC/\eta)
    \end{align*}

    Consequently, the number of rounds within $S$ on which no actions are predicted to be feasible is bounded by:
    \begin{align*}
        \left|\left\{ t \in S: \widehat\cA_t^{\bc,\eta-\fea} = \emptyset \right\}\right| \le J f_\beta(L_\cC/\eta)
    \end{align*}

    In particular, plugging in the guarantee from  Theorem \ref{thm:unbiased-algorithm-zero-margin} yields the following concrete form of $f_\beta(L_\cC/\eta)$, which holds with probability at least $1-\delta$:
    \begin{align*}
        f_\beta(L_\cC/\eta) = O\left( \left(\frac{L_\cC |S|^{1/4}}{\eta} + \frac{L_\cC^2}{\eta^2} \right) \ln(dJ|\cA||\cN||\cS|T/\delta) \right)
    \end{align*}
\end{lemma}

\begin{proof}
    The proof is similar to that of Lemma \ref{lem:infeasibility-calibration}.
    
    On any round $t$ where action $a$ is $\eta$-predicted to violate the $j$-th constraint, we have:
    \begin{align*}
        c_j(a,p_t) > \eta
    \end{align*}
    
    Since $a$ is in the benchmark class $\cA_S^{\bc,0}$, we have:
    \begin{align*}
        c_j(a,y_t) \le 0
    \end{align*}
    
    Combining these two facts gives $c_j(a,p_t)-c_j(a,y_t) > \eta$. Summing this difference over all rounds in $S$ where $a$ is predicted to $\eta$-violate the $j$-th constraint, we get:
    \begin{align*}
        \sum_{t=1}^T \1\left[t \in S, a \in \widehat\cA_t^{c_j,\eta-\infea}\right] (c_j(a,p_t)-c_j(a,y_t)) > \eta T^{c_j,S,\eta-\infea}(a)
    \end{align*}
    
    The left-hand side can be bounded using the properties of our predictions. By linearity and $L_\cC$-Lipschitzness of the constraint function, and by $(\cN,\cS,\eta,\beta)$-infeasibility calibration, we have that:
    \begin{align*}
        \MoveEqLeft \sum_{t=1}^T \1\left[t \in S, a \in \widehat\cA_t^{c_j,\eta-\infea}\right] (c_j(a,p_t)-c_j(a,y_t)) \\
        &= c_j\left(a, \sum_{t=1}^T p_t \1\left[t \in S, a \in \widehat\cA_t^{c_j,\eta-\infea}\right]\right) - c_j\left(a, \sum_{t=1}^T y_t \1\left[t \in S, a \in \widehat\cA_t^{c_j,\eta-\infea}\right]\right) \\
        &\le L_\cC \left\|\sum_{t=1}^T p_t \1\left[t \in S, a \in \widehat\cA_t^{c_j,\eta-\infea}\right] - \sum_{t=1}^T y_t \1\left[t \in S, a \in \widehat\cA_t^{c_j,\eta-\infea}\right] \right\|_\infty \\
        &\le L_\cC \beta(T^{c_j,S,\eta-\infea}(a))
    \end{align*}

    Combining these inequalities gives the first part of the lemma:    
    \begin{align*}
        \eta T^{c_j,S,\eta-\infea}(a) < L_\cC \beta(T^{c_j,S,\eta-\infea}(a))
    \end{align*}

    For the concrete bound, we substitute the explicit form for $\beta$ from Theorem \ref{thm:unbiased-algorithm-zero-margin}:
    \begin{align*}
        \eta f_\beta(L_\cC / \eta) = L_\cC O\left( \ln\frac{dJ|\cA||\cN||\cS|T}{\delta} \cdot |S|^{1/4} + \sqrt{\ln\frac{dJ|\cA||\cN||\cS|T}{\delta} \cdot f_\beta(L_\cC / \eta)} \right)
    \end{align*}

    Solve for $f_\beta(L_\cC / \eta)$ yields the stated form:
    \begin{align*}
        f_\beta(L_\cC / \eta) = O\left( \frac{L_\cC |S|^{1/4} \ln(dJ|\cA||\cN||\cS|T/\delta)}{\eta} + \frac{L_\cC^2 \ln(dJ|\cA||\cN||\cS|T/\delta)}{\eta^2} \right)
    \end{align*}

    Fix any arbitrary benchmark action $a \in \cA_S^{\bc,0}$. If on round $t \in S$, no action is predicted to be feasible, then it must be that $a$ is predicted to be infeasible. Hence, there must exist at least one constraint that $a$ is predicted to violate. As a result, we have:
    \begin{align*}
        \left|\left\{ t \in S: \widehat\cA_t^{\bc,\eta-\fea} = \emptyset \right\}\right| &\le \left|\left\{ t \in S: \exists j \in [J], c_j(a,p_t) > \eta \right\}\right| \\
        &\le \sum_{j=1}^J T^{c_j,S,\eta-\infea}(a) \\
        &\le J f_\beta(L_\cC / \eta)
    \end{align*}
\end{proof}

\subsubsection{Bounding the Cumulative Constraint Violation}

The bound on the cumulative constraint violation follows the same argument as in Theorem \ref{thm:CCV}. The key difference is that the agent's relaxed decision rule, which tolerates violations up to $\eta$ at each round, introduces an additional error term $\eta |S|$.
\begin{theorem} \label{thm:CCV-zero-margin}
    Let $\cS$ be a collection of subsequences. 
    Let $\cN$ be a set of agents, where each agent is equipped with a utility function $u: \cA \times \cY \to [0,1]$ and $J$ constraint functions $\{c_j: \cA \times \cY \to [-1,1]\}_{j \in [J]}$. 
    Let $\eta>0$ be the feasibility tolerance.
    Suppose each agent plays $\eta$-constrained best responses to $p_t$ to compete with actions from the benchmark class $\cA_S^{\bc,0}$ over each subsequence $S \in \cS$. If the sequence of predictions $p_1,\ldots,p_T$ is $(\cN,\cS,\eta,\alpha)$-decision calibrated and $(\cN,\cS,\eta,\beta)$-infeasibility calibrated, then the cumulative constraint violation of any agent over any subsequence $S \in \cS$ is bounded by:
    \[
        \CCV(S) \le L_\cC |\cA| \alpha(|S|/|\cA|) + J f_\beta(L_\cC/\eta) + \eta |S|
    \]
    In particular, plugging in the guarantee from  Theorem \ref{thm:unbiased-algorithm-zero-margin} yields the following bound, which holds with probability at least $1-\delta$:
    \[
        \CCV(S) \le O\left( \left( L_\cC |\cA| |S|^{1/4} + L_\cC \sqrt{|\cA||S|} + \frac{J L_\cC |S|^{1/4}}{\eta} + \frac{J L_\cC^2}{\eta^2} \right) \ln(dJ|\cA||\cN||\cS|T/\delta) + \eta|S| \right)
    \]
\end{theorem}

By setting $\eta = |T|^{-1/3}$, we arrive at the following concrete bound for the cumulative constraint violation.
\begin{corollary} \label{cor:CCV}
    Let $\cS$ be a collection of subsequences. Let $\cN$ be a set of agents, where each agent is equipped with a utility function $u: \cA \times \cY \to [0,1]$ and $J$ constraint functions $\{c_j: \cA \times \cY \to [-1,1]\}_{j \in [J]}$. 
    Let $\eta>0$ be the feasibility tolerance.
    Suppose each agent plays $\eta$-constrained best responses to $p_t$ to compete with actions from the benchmark class $\cA_S^{\bc,0}$ over each subsequence $S \in \cS$. The sequence of predictions $p_1,\ldots,p_T$ produced by \textsc{Decision-Infeasibility-Calibration} guarantees that with probability at least $1-\delta$, the cumulative constraint violation of any agent over any subsequence $S \in \cS$ is bounded by:
    \begin{align*}
        \CCV(S) \le O\left( \left( L_\cC |\cA| |S|^{1/4} + L_\cC \sqrt{|\cA||S|} + J (L_\cC+L_\cC^2) T^{2/3} \right) \ln(dJ|\cA||\cN||\cS|T/\delta) \right)
    \end{align*}
\end{corollary}

\subsubsection{Bounding the Regret}

The bound on the constrained swap regret follows the same argument as in Theorem \ref{thm:swap-regret}.
\begin{theorem} \label{thm:swap-regret-zero-margin}
    Let $\cS$ be a collection of subsequences. Let $\lambda: \bN \to (0,\infty)$ be a margin function. Let $\cN$ be a set of agents, where each agent is equipped with a utility function $u: \cA \times \cY \to [0,1]$ and $J$ constraint functions $\{c_j: \cA \times \cY \to [-1,1]\}_{j \in [J]}$. 
    Let $\eta>0$ be the feasibility tolerance.
    Suppose each agent plays $\eta$-constrained best responses to $p_t$ to compete with actions from the benchmark class $\cA_S^{\bc,0}$ over each subsequence $S \in \cS$. If the sequence of predictions $p_1,\ldots,p_T$ is $(\cN,\cS,\eta,\alpha)$-decision calibrated and $(\cN,\cS,\eta,\beta)$-infeasibility calibrated, then the constrained swap regret of any agent over any subsequence $S \in \cS$ is bounded by:
    \begin{align*}
        \Reg_\swap(u,\bc,0,S) \le 2 L_\cU |\cA| \alpha(|S|/|\cA|) + J |\cA| f_\beta(L_\cC/\eta)
    \end{align*}
    In particular, plugging in the guarantee from  Theorem \ref{thm:unbiased-algorithm-zero-margin} yields the following bound, which holds with probability at least $1-\delta$:
    \begin{align*}
        \Reg_\swap(u,\bc,0,S) \le O\left( \left( L_\cU |\cA| |S|^{1/4} + L_\cU \sqrt{|\cA||\cS|} + \frac{J L_\cC |\cA| |S|^{1/4}}{\eta} + \frac{J L_\cC^2 |\cA|}{\eta^2} \right) \ln(dJ|\cA||\cN||\cS|T/\delta) \right)
    \end{align*}
\end{theorem}

By setting $\lambda = T^{-1/3}$, we arrive at the following concrete bound for the constrained swap regret.
\begin{corollary} \label{cor:swap-regret-zero-margin}
    Let $\cS$ be a collection of subsequences. Let $\cN$ be a set of agents, where each agent is equipped with a utility function $u: \cA \times \cY \to [0,1]$ and $J$ constraint functions $\{c_j: \cA \times \cY \to [-1,1]\}_{j \in [J]}$. 
    Let $\eta>0$ be the feasibility tolerance.
    Suppose each agent plays $\eta$-constrained best responses to $p_t$ to choose action $a_t$. The sequence of predictions $p_1,\ldots,p_T$ produced by \textsc{Decision-Infeasibility-Calibration-Relaxed} guarantees that with probability at least $1-\delta$, the constrained swap regret against the benchmark class $\cA_S^{\bc,0}$ of any agent over any subsequence $S \in \cS$ is bounded by:
    \begin{align*}
        \Reg_\swap(u,\bc,0,S) \le O\left( \left( L_\cU |\cA| |S|^{1/4} + L_\cU \sqrt{|\cA||\cS|} + J (L_\cC+L_\cC^2) |\cA| T^{2/3} \right) \ln(dJ|\cA||\cN||\cS|T/\delta) \right)
    \end{align*}
\end{corollary}

We then achieve low constrained swap adaptive regret by instantiating our framework with the collection of all contiguous intervals, $\cS = \{[t_1,t_2]: 1 \le t_1 \le t_2 \le T\}$.
\begin{corollary} \label{cor:swap-adaptive-regret-zero-margin}
    Let $\cN$ be a set of agents, where each agent is equipped with a utility function $u: \cA \times \cY \to [0,1]$ and $J$ constraint functions $\{c_j: \cA \times \cY \to [-1,1]\}_{j \in [J]}$. 
    Let $\eta>0$ be the feasibility tolerance.
    Suppose each agent plays $\eta$-constrained best responses to $p_t$ to choose action $a_t$. 
    The sequence of predictions $p_1,\ldots,p_T$ produced by \textsc{Decision-Infeasibility-Calibration-Relaxed} guarantees that with probability at least $1-\delta$, the constrained swap adaptive regret of any agent is bounded by:
    \begin{align*}
        \Reg_{\swap-\adapt}(u,\bc,0) \le O\left( \left( ( L_\cU + J L_\cC + J L_\cC^2) |\cA| T^{2/3} \right) \ln(dJ|\cA||\cN||\cS|T/\delta) \right)
    \end{align*}
\end{corollary}

A dynamic benchmark with $\Delta$ changes partitions the entire time horizon into $\Delta+1$ intervals. By summing our per-subsequence regret bound over this specific partition, we obtain the following dynamic regret guarantee.
\begin{corollary} \label{cor:swap-dynamic-regret-zero-margin}
    Let $\cN$ be a set of agents, where each agent is equipped with a utility function $u: \cA \times \cY \to [0,1]$ and $J$ constraint functions $\{c_j: \cA \times \cY \to [-1,1]\}_{j \in [J]}$. 
    Let $\eta>0$ be the feasibility tolerance.
    Suppose each agent plays $\eta$-constrained best responses to $p_t$ to choose action $a_t$. The sequence of predictions $p_1,\ldots,p_T$ produced by \textsc{Decision-Infeasibility-Calibration-Relaxed} guarantees that with probability at least $1-\delta$, the constrained swap dynamic regret of any agent against any piecewise feasible sequence of action modification rule $\vec\phi \in (\cA^\cA)^T$ is bounded by:
    \begin{align*}
        \Reg_{\swap-\dyn}(u,\vec\phi) \le O\left( \left( (L_\cU + J L_\cC + J L_\cC^2) |\cA| T^{2/3} \Delta(\vec\phi) \right) \ln(dJ|\cA||\cN||\cS|T/\delta) \right)
    \end{align*}
    A benchmark sequence $\vec\phi$ is piecewise feasible if on each interval of constancy $I_k$, the corresponding rule $\psi_k$ maps to the set of actions that are feasible over that entire interval, i.e., $\psi_k: \cA \to \cA_{I_k}^{\bc,0}$.
\end{corollary}

\section{Unbiased Prediction Algorithm} \label{app:unbiased-prediction}
In this section we present the $\textsc{Unbiased-Prediction}$ algorithm of \citet{noarov2023highdimensional}.
We first introduce several notations and concepts from \citet{noarov2023highdimensional}. Let $\Pi = \{(x,p,y) \in \cX \times \cY \times \cY\}$ denote the set of possible realized triples at each round. An interaction over $T$ rounds produces a transcript $\pi_T \in \Pi^T$. 
We write $\pi^{<t}_T$ as the prefix of the first $t-1$ triples in $\pi_T$, for any $t \le T$. 
An \textit{event} $E\in\cE$ is a mapping from contexts and predictions to $[0,1]$ , i.e. $E: \cX \times \cY \to [0,1]$. 

The $\textsc{Unbiased-Prediction}$ algorithm makes predictions that are unbiased conditional on a collection of events $\cE$. The algorithm's conditional bias guarantee depends logarithmically on the number of events:

\begin{theorem}\citep{noarov2023highdimensional} \label{thm:unbiased-prediction-algorithm}
For a collection of events $\cE$ and convex prediction/outcome space $\cY\subseteq [0,1]^d$, Algorithm \ref{alg:unbiased-prediction} outputs, on any $T$-round transcript $\pi_T$, a sequence of distributions over predictions $\psi_1,...,\psi_T \in \Delta \cY$ such that for any $E \in \cE$:
    \[
    \left\| \sum_{t=1}^T \E_{p_t\sim\psi_t}[E(x_t, p_t)(p_t - y_t)] \right\|_\infty \leq O\left( \ln(d|\cE|T) + \sqrt{\ln(d|\cE|T) \cdot \sum_{t=1}^T \E_{p_t\sim\psi_t}[E(x_t,p_t)] } \right).
    \]
    The algorithm can be implemented with per-round running time scaling polynomially in $d$ and $|\cE|$.
\end{theorem}

\begin{algorithm}[H]
    \For{$t=1$ \KwTo $T$}{
        Observe $x_t$\;
        
        Define the distribution $q_t \in \Delta [2d|\cE|]$ such that for $E \in \cE, i\in[d], \sigma\in \{\pm 1\}$,
        \[
        q_t^{E, i, \sigma} \propto \exp\left( \frac{\eta}{2} \sum_{s=1}^{t-1} \sigma \cdot \E_{p_s\sim\psi_s}[E(x_s, p_s) (p_s^i - y_s^i)] \right);
        \]

        Output the solution to the minmax problem:
        \[
        \psi_t \gets \argmin_{\psi_t' \in \Delta\cY} \max_{y \in \cY} \E_{p_t \sim \psi_t'}\left[\sum_{E, i, \sigma} q_t^{E, i, \sigma} \cdot \sigma \cdot E(x_t, p_t) \cdot (p_t^i - y_t^i) \right];
        \]        
    }    
    \caption{$\textsc{Unbiased-Prediction}$}
    \label{alg:unbiased-prediction}
\end{algorithm}

Theorem \ref{thm:unbiased-prediction-algorithm} is stated as expected error bounds over randomized predictions $\psi_t\in\Delta\cY$. In the following Corollary \ref{cor:unbiased-prediction-algorithm-subsequence}, we state the guarantee based on realized predictions $p_t$ that are sampled from $\psi_t$, and generalize it to our multi-subsequence framework. Our guarantees in Theorem \ref{thm:unbiased-algorithm} directly follow from Corollary \ref{cor:unbiased-prediction-algorithm-subsequence}.

\begin{corollary} \label{cor:unbiased-prediction-algorithm-subsequence}
Let $\cS$ be a collection of subsequences. For a collection of events $\cE$ and convex prediction/outcome space $\cY\subseteq [0,1]^d$, Algorithm \ref{alg:unbiased-prediction} instantiated with the event collection $\{\1[t \in S] \cdot E\}_{S \in \cS, E \in \cE}$ outputs, on any $T$-round transcript $\pi_T$, a sequence of predictions $p_1,...,p_T \in \cY$ satisfying that, with probability at least $1-\delta$:
    \[
        \left\| \sum_{t \in S} E(x_t, p_t)(p_t - y_t) \right\|_\infty \leq O\left( \ln(d|\cE||\cS|T/\delta) \cdot |S|^{1/4} + \sqrt{\ln(d|\cE||\cS|T/\delta) \cdot \sum_{t \in S} E(x_t,p_t) } \right).
    \]
    The algorithm can be implemented with per-round running time scaling polynomially in $d$, $|\cE|$, and $|\cS|$.
\end{corollary}
\begin{proof}
    Fix any $m \in [d]$, $E \in \cE$, and $S \in \cS$. Consider the sequence $\{E(x_t, p_t)(p_{t,m} - y_{t,m}) - \E_{p_t\sim\psi_t}[E(x_t, p_t)(p_{t,m} - y_{t,m})]\}_{t=1}^T$, where $p_{t,m}$ and $y_{t,m}$ are the $m$-th coordinate of $p_t$ and $y_t$, respectively. It is a sequence of martingale differences, since for any $t \in [T]$: 
    \begin{align*}
        \E\left[ E(x_t, p_t)(p_{t,m} - y_{t,m}) - \E_{p_t\sim\psi_t}[E(x_t, p_t)(p_{t,m} - y_{t,m})] \mid \sigma(\pi_T^{<t},x_t)\right] = 0.
    \end{align*}

    The subsequence of these terms corresponding to rounds $t \in S$, i.e., $\{E(x_t, p_t)(p_{t,m} - y_{t,m}) - \E_{p_t\sim\psi_t}[E(x_t, p_t)(p_{t,m} - y_{t,m})]\}_{t \in S}$, is also a martingale difference sequence, because the selection rule is predictable with respect to the filtration $\sigma(\pi_T^{<t},x_t)$.

    By Freedman's inequality (Lemma \ref{lem:freedman}), we have that with probability at least $1-\frac{\delta}{d|\cE||\cS|}$:
    \begin{align*}
       \MoveEqLeft \left| \sum_{t \in S} E(x_t, p_t)(p_{t,m} - y_{t,m}) - \sum_{t \in S} \E_{p_t\sim\psi_t}[E(x_t, p_t)(p_{t,m} - y_{t,m})] \right| \\
       &\le O \Bigg( \sqrt{ \ln(d|\cE||\cS|\ln(|S|)/\delta) \cdot \sum_{t \in S} \E\left[ \left( E(x_t, p_t)(p_{t,m} - y_{t,m}) - \E_{p_t\sim\psi_t}[E(x_t, p_t)(p_{t,m} - y_{t,m})] \right)^2 \mid \sigma(\pi_T^{<t},x_t) \right] } \\
       &\quad + \ln(d|\cE||\cS|\ln(|S|)/\delta) \Bigg) \\
       &\le O\left( \sqrt{ \ln(d|\cE||\cS|/\delta) \cdot \sum_{t \in S} \E\left[ \left( E(x_t, p_t)(p_{t,m} - y_{t,m}) \right)^2 \mid \sigma(\pi_T^{<t},x_t) \right] } + \ln(d|\cE||\cS|/\delta) \right) \\
       &\le O\left( \sqrt{ \ln(d|\cE||\cS|\ln(|S|)/\delta) \cdot \sum_{t \in S} \E\left[ E(x_t, p_t) \mid \sigma(\pi_T^{<t},x_t) \right] } + \ln(d|\cE||\cS|\ln(|S|)/\delta) \right)
    \end{align*}
    where the second inequality follows from the fact that the conditional variance is less than or equal to the conditional second moment, and the third inequality follows from the fact that $E(x_t,p_t) \in [0,1]$ and $|p_{t,m}-y_{t,m}| \in [0,1]$.

    Using the union bound over all $m \in [d]$, $E \in \cE$, and $S \in \cS$, we have that with probability at least $1-\delta$, for any $E \in \cE$ and any $S \in \cS$:
    \begin{align*}
       \MoveEqLeft \left\| \sum_{t \in S} E(x_t, p_t)(p_{t} - y_{t}) - \sum_{t \in S} \E_{p_t\sim\psi_t}[E(x_t, p_t)(p_{t} - y_{t})] \right\|_\infty \\
       &\le O\left( \sqrt{ \ln(d|\cE||\cS|\ln(|S|)/\delta) \cdot \sum_{t \in S} \E\left[ E(x_t, p_t) \mid \sigma(\pi_T^{<t},x_t) \right] } + \ln(d|\cE||\cS|\ln(|S|)/\delta) \right)
    \end{align*}

    By applying Theorem \ref{thm:unbiased-prediction-algorithm} with the event collection $\{\1[t \in S] \cdot E\}_{S \in \cS, E \in \cE}$ and combining its guarantee with the deviation bound above, we derive that with probability at least $1-\delta$, for any $E \in \cE$ and any $S \in \cS$:
    \begin{align*}
        \MoveEqLeft \left\| \sum_{t \in S} E(x_t, p_t)(p_t - y_t) \right\|_\infty \\
        &\le \left\| \sum_{t=1}^T \E_{p_t\sim\psi_t}[\1[t \in S]E(x_t, p_t)(p_t - y_t)] \right\|_\infty 
        + \left\| \sum_{t \in S} E(x_t, p_t)(p_t - y_t) - \sum_{t \in S} \E_{p_t\sim\psi_t}[E(x_t, p_t)(p_t - y_t)] \right\|_\infty \\
        &\leq O\left( \ln(d|\cE||S|T) + \sqrt{ \ln(d|\cE||S|T) \cdot \sum_{t \in S} \E_{p_t\sim\psi_t}[E(x_t,p_t)] } \right) \\
        &\quad + O\left( \sqrt{ \ln(d|\cE||\cS|\ln(|S|)/\delta) \cdot \sum_{t \in S} \E\left[ E(x_t, p_t) \mid \sigma(\pi_T^{<t},x_t) \right] } + \ln(d|\cE||\cS|\ln(|S|)/\delta) \right) \\
        &\le O \Bigg( \sqrt{ \ln(d|\cE||S|T/\delta) \cdot \sum_{t \in S} \E_{p_t\sim\psi_t}[E(x_t,p_t)] } + \sqrt{ \ln(d|\cE||\cS|T/\delta) \cdot \sum_{t \in S} \E\left[ E(x_t, p_t) \mid \sigma(\pi_T^{<t},x_t) \right] } \\
        &\quad + \ln(d|\cE||\cS|T/\delta) \Bigg).
    \end{align*}

    The next step is to bound the deviation of $\sum_{t \in S} E(x_t,p_t)$ from its expectation $\sum_{t \in S} \E_{p_t\sim\psi_t}[E(x_t,p_t)]$ and from its conditional-expectation-based version $\sum_{t \in S} \E\left[ E(x_t, p_t) \mid \sigma(\pi_T^{<t},x_t) \right]$ for any $S \in \cS$ and $E \in \cE$. We will again apply a martingale concentration inequality.

    Fix any $E \in \cE$ and $S \in \cS$. Consider the sequence $\{E(x_t, p_t) - \E_{p_t\sim\psi_t}[E(x_t, p_t)]\}_{t=1}^T$ and the sequence $\{E(x_t, p_t) - \E\left[ E(x_t, p_t) \mid \sigma(\pi_T^{<t},x_t) \right]\}_{t=1}^T$. Both of them are sequences of martingale differences, since for any $t \in [T]$: 
    \begin{align*}
        & \E\left[ E(x_t, p_t) - \E_{p_t\sim\psi_t}[E(x_t, p_t)] \mid \sigma(\pi_T^{<t},x_t) \right] = 0 \\
        & \E\left[ E(x_t, p_t) - \E\left[ E(x_t, p_t) \mid \sigma(\pi_T^{<t},x_t) \right] \mid \sigma(\pi_T^{<t},x_t) \right] = 0
    \end{align*}

    The subsequence of these terms corresponding to rounds $t \in S$, i.e., $\{E(x_t, p_t) - \E_{p_t\sim\psi_t}[E(x_t, p_t)]\}_{t \in S}$ and $\{E(x_t, p_t) - \E\left[ E(x_t, p_t) \mid \sigma(\pi_T^{<t},x_t) \right]\}_{t \in S}$, are also both martingale difference sequences, because the selection rule is predictable with respect to the filtration $\sigma(\pi_T^{<t},x_t)$.

    By Azuma-Hoeffding inequality (Lemma \ref{lem:azuma}), we have that with probability at least $1-\frac{\delta}{2d|\cE||\cS|}$:
    \begin{align*}
       & \left| \sum_{t \in S} E(x_t, p_t) - \sum_{t \in S} \E_{p_t\sim\psi_t}[E(x_t, p_t)] \right| \le 2\sqrt{2 \ln (4d|\cE||\cS| / \delta) \cdot |S|} \\
       & \left| \sum_{t \in S} E(x_t, p_t) - \sum_{t \in S} \E\left[ E(x_t, p_t) \mid \sigma(\pi_T^{<t},x_t) \right] \right| \le 2\sqrt{2 \ln (4d|\cE||\cS| / \delta) \cdot |S|}.
    \end{align*}

    Using the union bound over all $E \in \cE$ and $S \in \cS$, we have that with probability at least $1-\delta$, for any $E \in \cE$ and any $S \in \cS$:
    \begin{align*}
       & \left| \sum_{t \in S} E(x_t, p_t) - \sum_{t \in S} \E_{p_t\sim\psi_t}[E(x_t, p_t)] \right| \le O\left( \sqrt{ \ln (d|\cE||\cS| / \delta) \cdot |S|} \right) \\
       & \left| \sum_{t \in S} E(x_t, p_t) - \sum_{t \in S} \E\left[ E(x_t, p_t) \mid \sigma(\pi_T^{<t},x_t) \right] \right| \le O\left( \sqrt{ \ln (d|\cE||\cS| / \delta) \cdot |S|} \right)
    \end{align*}

    Finally, substituting the above concentration bound for $\sum_{t \in S} E(x_t, p_t)$ into our high-probability guarantee yields the final result, with probability at least $1-\delta$.
    \begin{align*}
        \MoveEqLeft \left\| \sum_{t \in S} E(x_t, p_t)(p_t - y_t) \right\|_\infty \\
        &\le O \left( \sqrt{ \ln(d|\cE||S|T/\delta) \cdot \sum_{t \in S} \left( E(x_t,p_t) + O\left( \sqrt{ \ln (d|\cE||\cS| / \delta) \cdot |S|} \right) \right) } + \ln(d|\cE||\cS|T/\delta) \right) \\
        &\le O\left( \ln(d|\cE||\cS|T/\delta) \cdot |S|^{1/4} + \sqrt{ \ln(d|\cE||S|T/\delta) \cdot \sum_{t \in S} E(x_t,p_t)} \right)
    \end{align*}
    
\end{proof}

\section{Azuma-Hoeffding's Inequality} \label{app:azuma}
\begin{lemma} \label{lem:azuma}
Let $Z_1, \ldots, Z_T$ be a martingale difference sequence. $\left|Z_t\right| \leq M$ for all $t$. Then with probability at least $1-\delta$:
$$
\left|\sum_{t=1}^T Z_t\right| \leq M \sqrt{2 T \ln \frac{2}{\delta}}
$$
\end{lemma}



\section{Freedman’s Inequality} \label{app:freedman}

The following lemma gives a standard form of Freedman's inequality, which can be found in works such as \cite{tropp2011freedman}.
\begin{lemma} \label{lem:freedman-standard}
    Let $(\cF_i)_{i=0}^n$ be a filtration. Let $Z_1, \dots, Z_n$ be a martingale difference sequence with respect to $(\cF_i)_{i=0}^n$. $Z_i \le M$ for all $i$. Let $V_n = \sum_{i=1}^{n} \mathbb{E}[Z_i^2 | \cF_{i-1}]$. Then, for all $\tau \ge 0$ and $v > 0$,
    \[
        P\left( \sum_{i=1}^n Z_i \ge \tau, V_n \le v \right) \le \exp\left( -\frac{\tau^2/2}{v+M\tau/3} \right)
    \]
\end{lemma}

We derive the following convenient form that we use in our proofs.
\begin{lemma} \label{lem:freedman}
    Let $(\cF_i)_{i=0}^n$ be a filtration. Let $Z_1, \dots, Z_n$ be a martingale difference sequence with respect to $(\cF_i)_{i=0}^n$. $Z_i \le M$ for all $i$. Let $V_n = \sum_{i=1}^{n} \mathbb{E}[Z_i^2 | \cF_{i-1}]$. For any $\delta \in (0,1)$, with probability at least $1-\delta$:
    \[
        \sum_{i=1}^{n} Z_i \le 2\sqrt{V_n \left( \ln(1/\delta)+C_{n,M} \right)} + \left( \frac{2}{3}M+3 \right) \left( \ln(1/\delta)+C_{n,M} \right)
    \]
    where $C_{n,M} = 2\ln(\ln(nM^2)+1) + 2\ln(2)$.
\end{lemma}
\begin{proof}
    Let $\tau = \sqrt{2v\ln(1/\delta)} + \frac{2}{3}M\ln(1/\delta)$, it satisfies that $\frac{\tau^2/2}{v+M\tau/3} > \ln(1/\delta)$, because:
    \begin{align*}
        \tau^2/2 - (v+M\tau/3)\ln(1/\delta) 
        &= v\ln(1/\delta) + \frac{2}{9}M^2\ln^2(1/\delta) + \sqrt{2v\ln(1/\delta)} \cdot \frac{2}{3}M\ln(1/\delta) \\
        &\quad - v\ln(1/\delta) - \frac{1}{3}M \left( \sqrt{2v\ln(1/\delta)} + \frac{2}{3}M\ln(1/\delta) \right) \ln(1/\delta) \\
        &= \sqrt{2v\ln(1/\delta)} \cdot \frac{2}{3}M\ln(1/\delta) - \frac{1}{3}M \cdot \sqrt{2v\ln(1/\delta)} \cdot \ln(1/\delta) \\
        &= \frac{1}{3}M \cdot \sqrt{2v\ln(1/\delta)} \cdot \ln(1/\delta) \\
        &> 0
    \end{align*}

    Applying Lemma \ref{lem:freedman-standard} with this choice of $\tau$, we derive that for any $\delta \in (0,1)$ and $v > 0$:
    \begin{align*}
        P\left( \sum_{i=1}^n Z_i \ge  \sqrt{2v\ln(1/\delta)} + \frac{2}{3}M\ln(1/\delta), \; V_n \le v \right) \le \delta
    \end{align*}

    Let $K = \lceil \log_2(nM^2) \rceil$. Let $\delta_k = \frac{6\delta}{\pi^2k^2}$ and $v_k = 2^k$ for $k = 1, \ldots, K$.
    Instantiate the above inequality with $\delta_k$ and $v_k$, we have for any $k \in [K]$:
    \begin{align*}
        P(E_k) \le \delta_k
    \end{align*}
    where $E_k$ denotes the event:
    \begin{align*}
        E_k = \left\{ \sum_{i=1}^n Z_i \ge  \sqrt{2v_k\ln(1/\delta_k)} + \frac{2}{3}M\ln(1/\delta_k), \; V_n \le v_k \right\}
    \end{align*}

    Applying the union bound, we have:
    \begin{align*}
        P\left( \cup_{k=1}^K E_k \right) &\le \sum_{k=1}^K P(E_k) \\
        &\le \sum_{k=1}^K \delta_k \\
        &= \sum_{k=1}^K \frac{6\delta}{\pi^2k^2} \\
        &< \delta
    \end{align*}

    Therefore, with probability at least $1-\delta$, none of the events $E_k$ occurs. We will prove that our desired guarantee holds true conditional on this high-probability event $\cap_{k=1}^K E_k^c$.

    For convenience, let $v_0 = 0$. Then $\{(v_{k-1},v_k]\}_{k=1}^K$ forms a partition of $(0,2^K]$. Since $0 < V_n \le nM^2 \le 2^K$, there must exist $k \in [1,K]$, such that $v_{k-1} < V_n \le v_k$.
    
    Since $E_k$ does not happen, it must be that $\sum_{i=1}^n Z_i < \sqrt{2v_k\ln(1/\delta_k)} + \frac{2}{3}M\ln(1/\delta_k)$.

    If $k \in [2,K]$, then $v_k = 2v_{k-1} < 2V_n$. Hence, $\sum_{i=1}^n Z_i < \sqrt{4V_n\ln(1/\delta_k)} + \frac{2}{3}M\ln(1/\delta_k)$.

    If $k = 1$, then $v_k = 2^k = 2$. Hence, $\sum_{i=1}^n Z_i < 2\sqrt{\ln(1/\delta_k)} + \frac{2}{3}M\ln(1/\delta_k)$.

    Therefore, conditional on the event $\cap_{k=1}^K E_k^c$, which happens with probability at least $1-\delta$, it always holds true that:
    \begin{align*}
        \sum_{i=1}^n Z_i < 2\sqrt{V_n\ln(1/\delta_k)} + 2\sqrt{\ln(1/\delta_k)} + \frac{2}{3}M\ln(1/\delta_k)
    \end{align*}

    By definition of $\delta_k$, we have $\ln(1/\delta_k) = \ln\frac{\pi^2k^2}{6\delta} \ge \ln\frac{\pi^2}{6} > 0.49$. Hence, $\sqrt{\ln(1/\delta_k)} > 0.7$. Then we have:
    \begin{align*}
        \sum_{i=1}^n Z_i &< 2\sqrt{V_n\ln(1/\delta_k)} + 3 \cdot 0.7 \cdot \sqrt{\ln(1/\delta_k)} + \frac{2}{3}M\ln(1/\delta_k) \\
        &< 2\sqrt{V_n\ln(1/\delta_k)} + \left( \frac{2}{3}M+3 \right) \ln(1/\delta_k)
    \end{align*}

    We note that
    \begin{align*}
        \ln(1/\delta_k) &= \ln\frac{\pi^2k^2}{6\delta} \\
        &= 2\ln\frac{\pi k}{\sqrt{6}} + \ln(1/\delta) \\
        &\le 2\ln\frac{\pi K}{\sqrt{6}} + \ln(1/\delta) \\
        &\le 2\ln\frac{\pi (\log_2(nM^2)+1)}{\sqrt{6}} + \ln(1/\delta) \\
        &\le 2\ln(2\ln(nM^2)+2) + \ln(1/\delta) \\
        &= C_{n,M} + \ln(1/\delta)
    \end{align*}

    Therefore, with probability at least $1-\delta$:
    \begin{align*}
        \sum_{i=1}^n Z_i &< 2\sqrt{V_n \left( \ln(1/\delta)+C_{n,M} \right)} + \left( \frac{2}{3}M+3 \right) \left( \ln(1/\delta)+C_{n,M} \right)
    \end{align*}
\end{proof}

\end{document}